\documentclass[a4paper,USenglish,autoref]{lipics-v2021}

\usepackage{cite}
\usepackage{subcaption} 

\usepackage{mathtools}

\newcommand{\Gtri}{\ensuremath{G_{\Delta}}}

\newif\ifcomment
\commenttrue    
\newcommand{\bluecomment}[1]{\ifcomment\color{blue} #1 \color{black}\fi}

\newif\ifconf
\conffalse

\newif\iffigabbrv
\figabbrvfalse   
\newcommand{\figtext}{\iffigabbrv Fig.\else Figure\fi}
\newcommand{\figstext}{\iffigabbrv Figs.\else Figures\fi}

\title{Deadlock and Noise in Self-Organized Aggregation Without Computation}
\titlerunning{Aggregation Without Computation}

\author{Joshua J. Daymude}{Biodesign Center for Biocomputing, Security and Society, Arizona State University, Tempe, AZ}{jdaymude@asu.edu}{https://orcid.org/0000-0001-7294-5626}{}
\author{Noble C. Harasha}{The Peggy Payne Academy at McClintock High School, Tempe, AZ}{nharasha@mit.edu}{}{}
\author{Andr\'ea W. Richa}{School of Computing and Augmented Intelligence, Arizona State University, Tempe, AZ}{aricha@asu.edu}{https://orcid.org/0000-0003-3592-3756}{}
\author{Ryan Yiu}{School of Computing and Augmented Intelligence, Arizona State University, Tempe, AZ}{ryiu012@gmail.com}{}{}
\authorrunning{J.\ J.\ Daymude, N.\ C.\ Harasha, A.\ W.\ Richa, and R.\ Yiu}

\Copyright{Joshua J.\ Daymude, Noble C.\ Harasha, Andr\'ea W.\ Richa, and Ryan Yiu}



\ifconf
\relatedversion{A full version of this paper is available online at \bluecomment{\url{TODO}}.}
\else\fi

\supplement{Source code for the continuous-space simulations is openly available at \url{https://github.com/SOPSLab/SwarmAggregation}.
Source code for the discrete adaptation is available in AmoebotSim~\cite{Daymude2021-amoebotsim} (\url{https://github.com/SOPSLab/AmoebotSim}), a visual simulator for the amoebot model of programmable matter.}

\funding{The authors gratefully acknowledge their support from the U.S.\ ARO under MURI award \#W911NF-19-1-0233 and from the Arizona State University Biodesign Institute.}

\acknowledgements{We thank Dagstuhl~\cite{Berman2019-dagstuhlprogmat} for hosting the seminar that inspired this research, Roderich Gro\ss\ for introducing us to this open problem, and Aaron Becker and Dan Halperin for their contributions to our understanding of the continuous setting.}

\ifconf\else
\hideLIPIcs
\nolinenumbers
\fi

\begin{document}

\maketitle

\begin{abstract}
Aggregation is a fundamental behavior for swarm robotics that requires a system to gather together in a compact, connected cluster.
In 2014, Gauci et al.\ proposed a surprising algorithm that reliably achieves swarm aggregation using only a binary line-of-sight sensor and no arithmetic computation or persistent memory.
It has been rigorously proven that this algorithm will aggregate one robot to another, but it remained open whether it would always aggregate a system of $n > 2$ robots as was observed in experiments and simulations.
We prove that there exist deadlocked configurations from which this algorithm cannot achieve aggregation for $n > 3$ robots when the robots' motion is uniform and deterministic.
On the positive side, we show that the algorithm (\textit{i}) is robust to small amounts of error, enabling deadlock avoidance, and (\textit{ii}) provably achieves a linear runtime speedup for the $n = 2$ case when using a cone-of-sight sensor.
Finally, we introduce a noisy, discrete adaptation of this algorithm that is more amenable to rigorous analysis of noise and whose simulation results align qualitatively with the original, continuous algorithm.
\end{abstract}

\section{Introduction} \label{sec:intro}

The fields of swarm robotics~\cite{Brambilla2013-swarmrobotics,Hamann2018-swarmrobotics,Dorigo2020-futureswarms,Dorigo2021-swarmrobotics} and programmable matter~\cite{Angluin2006-populationprotocols,Woods2013-nubot,Flocchini2019-mobileentities} seek to engineer systems of simple, easily manufactured robot modules that can cooperate to perform tasks involving collective movement and reconfiguration.
Our present focus is on the \textit{aggregation problem} (also referred to as ``gathering''~\cite{Cieliebak2003-gathering,Flocchini2005-gatheringasync,Fates2010-gathering} and ``rendezvous''~\cite{Cortes2006-robustrendezvous,Zebrowski2007-energyrendezvous,Yu2012-rendezvous}) in which a robot swarm must gather together in a compact, connected cluster~\cite{Bayindir2016-swarmtasks}.
Aggregation has a rich history in swarm robotics as a prerequisite for other collective behaviors requiring densely connected swarms.
Inspired by self-organizing aggregation in nature~\cite{Devreotes1989-dictyostelium,Deneubourg1990-barkbeetle,Magurran1990-fishschooling,Camazine2001-selforgbio,Jeanson2005-cockroachaggregation,Mlot2011-fireantrafts}, numerous approaches for swarm aggregation have been proposed, each one seeking to achieve aggregation faster, more robustly, and with less capable individuals than the last~\cite{Agrawal2017-tunablestructures,Deblais2018-boundarycontrol,Firat2020-informedaggregation,Li2021-bobbots,Misir2021-dynamicaggregation}.

One goal from the theoretical perspective has been to identify \textit{minimal capabilities} for an individual robot such that a collective can provably accomplish a given task.
Towards this goal, Roderich Gro\ss\ and others at the Natural Robotics Laboratory have developed a series of very simple algorithms for swarm behaviors like spatially sorting by size~\cite{Gross2009-segregation1,Chen2012-segregation2}, aggregation~\cite{Gauci2014-aggregation}, consensus~\cite{Ozedmir2018-consensus}, and coverage~\cite{Ozdemir2019-coverage}.
These algorithms use at most a few bits of sensory information and express their entire structure as a single ``if-then-else'' statement, avoiding any arithmetic computation or persistent memory.
Although these algorithms have been shown to perform well in both robotic experiments and simulations with larger swarms, some lack general, rigorous proofs that guarantee the correctness of the swarm's behavior.

In this work, we investigate the Gauci et al.\ swarm aggregation algorithm~\cite{Gauci2014-aggregation} (summarized in Section~\ref{sec:aggregation}) whose provable convergence for systems of $n > 2$ robots remained an open question.
In Section~\ref{sec:negative}, we answer this question negatively, identifying deadlocked configurations from which aggregation is never achieved.
Motivated by the need to break these deadlocks, we corroborate and extend the simulation results of~\cite{Gauci2014-phdthesis} by showing that the algorithm is robust to two distinct forms of error (Section~\ref{sec:robust}).
Additionally, we prove that the time required for a single robot to aggregate to a static robot improves by a linear factor when using a cone-of-sight sensor instead of a line-of-sight sensor; however, simulations show this comparative advantage decreases for larger swarms (Section~\ref{sec:cone}).
Finally, in an effort to analyze this algorithm under an explicit modeling of noise --- as opposed to the noise implicit from the natural physics of robot collisions and slipping --- we introduce a noisy, discrete adaptation in Section~\ref{sec:discrete}.
Simulations of this discrete adaptation align qualitatively with the original algorithm in the continuous setting, but unfortunately exhibit behavior that is similarly difficult to analyze theoretically.

\section{The Gauci et al.\ Swarm Aggregation Algorithm} \label{sec:aggregation}

Given $n$ robots in arbitrary initial positions on the two-dimensional plane, the goal of the \textit{aggregation problem} is to define a controller that, when used by each robot in the swarm, eventually forms a compact, connected cluster.
Gauci et al.~\cite{Gauci2014-aggregation} introduced an algorithm for aggregation among e-puck robots~\cite{Mondada2009-epuck} that only requires binary information from a robot's (infinite range) line-of-sight sensor indicating whether it sees another robot ($I = 1$) or not ($I = 0$).
The controller $x = (v_{\ell 0}, v_{r0}, v_{\ell 1}, v_{r1}) \in [-1, 1]^4$ actuates the left and right wheels according to velocities $(v_{\ell 0}, v_{r0})$ if $I = 0$ and $(v_{\ell 1}, v_{r1})$ otherwise.
Using a grid search over a sufficiently fine-grained parameter space and evaluating performance according to a dispersion metric, they determined that the highest performant controller was:
\[x^* = (-0.7, -1, 1, -1).\]
Thus, when no robot is seen, a robot using $x^*$ will rotate around a point $c$ that is $90^\circ$ counter-clockwise from its line-of-sight sensor and $R = 14.45$ cm away at a speed of $\omega_0 = -0.75$ rad/s; when a robot is seen, it will rotate clockwise in place at a speed of $\omega_1 = -5.02$ rad/s.
The following three theorems summarize the theoretical results for this aggregation algorithm.

\begin{theorem}[Gauci et al.~\cite{Gauci2014-aggregation}] \label{thm:infiniterange}
    If the line-of-sight sensor has finite range, then for every controller $x$ there exists an initial configuration in which the robots form a connected visibility graph but from which aggregation will never occur.
\end{theorem}

\begin{theorem}[Gauci et al.~\cite{Gauci2014-aggregation}] \label{thm:staticaggregate}
    One robot using controller $x^*$ will always aggregate to another static robot or static circular cluster of robots.
\end{theorem}

\begin{theorem}[Gauci et al.~\cite{Gauci2014-aggregation}] \label{thm:twoaggregate}
    Two robots both using controller $x^*$ will always aggregate.
\end{theorem}

Our main goal, then, is to investigate the following conjecture that is well-supported by evidence from simulations and experiments.

\begin{conjecture} \label{conj:naggregate}
    A system of $n > 2$ robots each using controller $x^*$ will always aggregate.
\end{conjecture}

Throughout the remaining sections, we measure the degree of aggregation in the system using the following metrics:
\begin{itemize}
    \item \textit{Smallest Enclosing Disc Circumference.} The smallest enclosing disc of a set of points $S$ in the plane is the circular region of the plane containing $S$ and having the smallest possible radius.
    Smaller circumferences correspond to more aggregated configurations.
    
    \item \textit{Convex Hull Perimeter.} The convex hull of a set of points $S$ in the plane is the smallest convex polygon enclosing $S$.
    Smaller perimeters correspond to more aggregated configurations.
    Due to the flexibility of convex polygons, this metric is less sensitive to outliers than the smallest enclosing disc which is forced to consider a circular region.
    
    \item \textit{Dispersion (2nd Moment).} Adapting Gauci et al.~\cite{Gauci2014-aggregation} and Graham and Sloane~\cite{Graham1990-pennypacking}, let $p_i$ denote the $(x,y)$-coordinate of robot $i$ on the continuous plane and $\overline{p} = \frac{1}{n}\sum_{i=1}^np_i$ be the centroid of the system.
    Dispersion is defined as:
    \[\sum_{i=1}^n||p_i - \overline{p}||_2 = \sum_{i=1}^n\sqrt{(x_i - \overline{x})^2 + (y_i - \overline{y})^2}\]
    Smaller values of dispersion correspond to more aggregated configurations.
    
    \item \textit{Cluster Fraction.} A cluster is defined as a set of robots that is ``connected'' by means of (nearly) touching.
    Following Gauci et al.~\cite{Gauci2014-aggregation}, our final metric for aggregation is the fraction of robots in the largest cluster.
    Unlike the previous metrics, larger cluster fractions correspond to more aggregated configurations.
\end{itemize}

We use dispersion as our primary metric of aggregation since it is the metric that is least sensitive to outliers and was used by Gauci et al.~\cite{Gauci2014-aggregation}, enabling a clear comparison of results.

\section{Impossibility of Aggregation for More Than Three Robots} \label{sec:negative}

In this section, we rigorously establish a negative result indicating that Conjecture~\ref{conj:naggregate} does not hold in general.
This result identifies a deadlock that, in fact, occurs for a large class of controllers that $x^*$ belongs to.
We say a controller $x = (v_{\ell 0}, v_{r0}, v_{\ell 1}, v_{r1}) \in [-1, 1]^4$ is \textit{clockwise-searching} if $v_{r0} < v_{\ell 0} < 0$.
In other words, a clockwise-searching controller maps $I = 0$ (i.e., the case in which no robot is detected by the line-of-sight sensor) to a clockwise rotation about the center of rotation $c$ that is a distance $R > 0$ away.\footnote{Note that an analogous version of Theorem~\ref{thm:deadlock} would hold for counter-clockwise-searching controllers if a robot's center of rotation was $90^\circ$ clockwise rather than counter-clockwise from its line-of-sight sensor.}

\begin{theorem} \label{thm:deadlock}
    For all $n > 3$ and all clockwise-searching controllers $x$, there exists an initial configuration of $n$ robots from which the system will not aggregate when using controller $x$.
\end{theorem}
\begin{proof}
    At a high level, we construct a deadlocked configuration by placing the $n$ robots in pairs such that no robot sees any other robot with its line-of-sight sensor --- implying that all robots continually try to rotate about their centers of rotation --- and each pair's robots mutually block each other's rotation.
    This suffices for the case that $n$ is even; when $n$ is odd, we extend the all-pairs configuration to include one mutually blocking triplet.
    Thus, no robots can move in this configuration since they are all mutually blocking, and since no robot sees any other they remain in this disconnected (non-aggregated) configuration indefinitely.
    
    \begin{figure}[t]
        \centering
        \begin{subfigure}{.9\textwidth}
            \centering
            \includegraphics[scale=0.35]{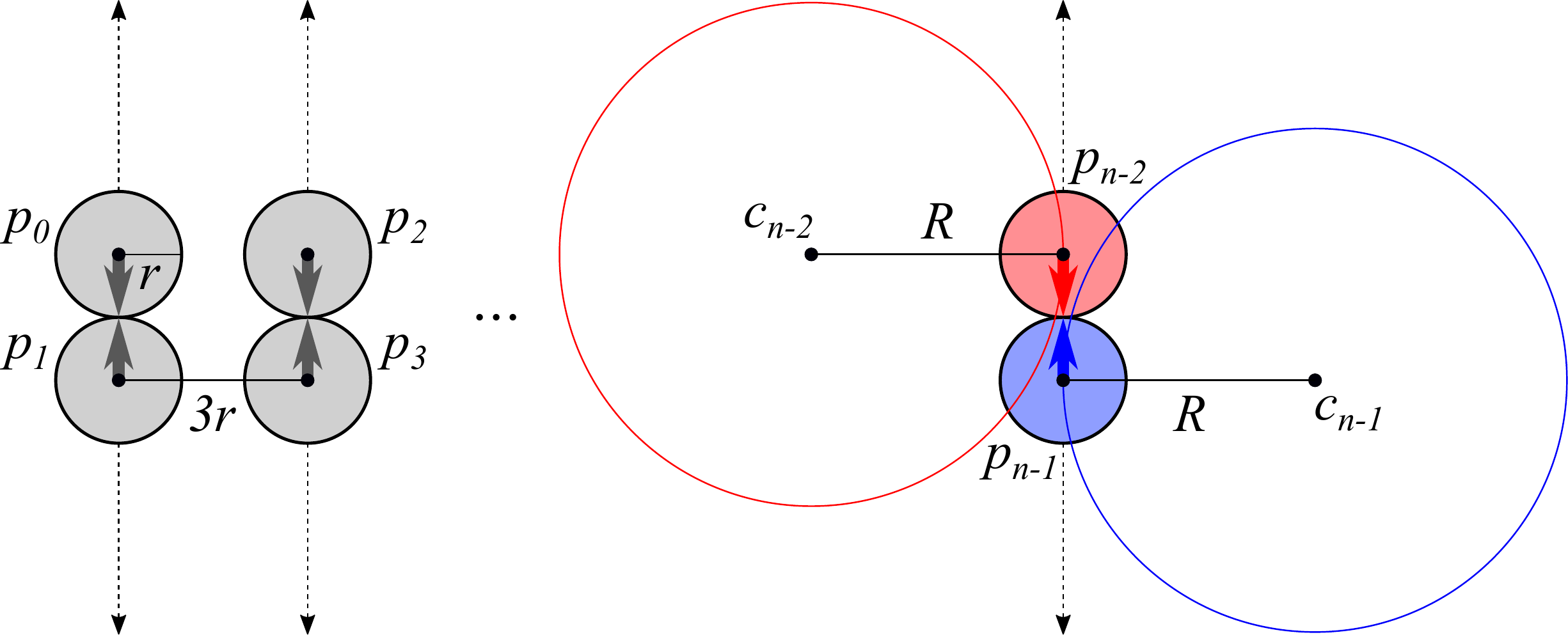}
            \caption{\centering}
            \label{fig:deadlockeven}
        \end{subfigure} \\
        \begin{subfigure}{.9\textwidth}
            \centering
            \includegraphics[scale=0.35]{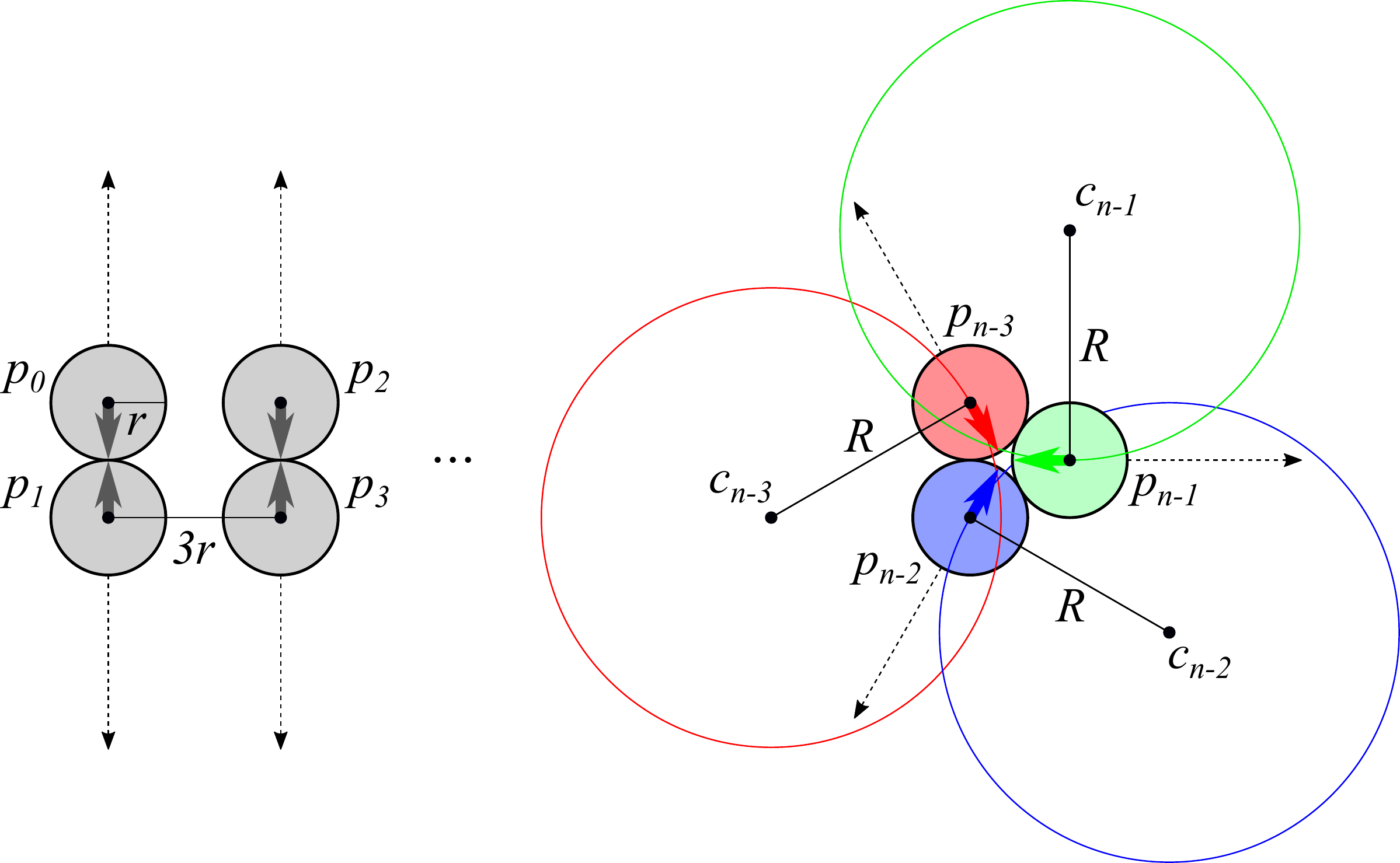}
            \caption{\centering}
            \label{fig:deadlockodd}
        \end{subfigure}
        \caption{The deadlocked configuration described in the proof of Theorem~\ref{thm:deadlock} for \textbf{\textsf{(a)}} $n > 3$ even and \textbf{\textsf{(b)}} $n > 3$ odd that would remain disconnected and non-compact indefinitely.}
        \label{fig:deadlock}
    \end{figure}
    
    In detail, first suppose $n > 3$ is even.
    As in~\cite{Gauci2014-aggregation}, let $r$ denote the radius of a robot.
    For each $i \in \{0, 1, \ldots, \frac{n}{2}-1\}$, place robots $p_{2i}$ and $p_{2i+1}$ at points $(3r \cdot i, r)$ and $(3r \cdot i, -r)$, respectively.
    Orient all robots $p_{2i}$ with their line-of-sight sensors in the $+y$ direction, and orient all robots $p_{2i+1}$ in the $-y$ direction.
    This configuration is depicted in \figtext~\ref{fig:deadlockeven}.
    Due to their orientations, no robot can see any others; thus, since $x$ is a clockwise-searching controller, all robots $p_{2i}$ are attempting to move in the $-y$ direction while all robots $p_{2i+1}$ are attempting to move in the $+y$ direction.
    Each pair of robots is mutually blocking, resulting in no motion.
    Moreover, since each consecutive pair of robots has a horizontal gap of distance $r$ between them, this configuration is disconnected and thus non-aggregated.
    
    It remains to consider when $n > 3$ is odd.
    Organize the first $n-3$ robots in pairs according to the description above; since $n$ is odd, we have that $n-3$ must be even.
    Then place robot $p_{n-1}$ at point $(3r(\frac{n}{2} - 1) + \sqrt{3}r, 0)$ with its line-of-sight sensor oriented at $0^\circ$ (i.e., the $+x$ direction), robot $p_{n-2}$ at point $(3r(\frac{n}{2} - 1), -r)$ with orientation $240^\circ$, and robot $p_{n-3}$ at point $(3r(\frac{n}{2} - 1), r)$ with orientation $120^\circ$, as depicted in \figtext~\ref{fig:deadlockodd}.
    By a nearly identical argument to the one above, this configuration will also remain deadlocked and disconnected.
    
    Therefore, we conclude that in all cases there exists a configuration of $n$ robots from which no clockwise-searching controller could achieve aggregation.
\end{proof}

We have shown that no clockwise-searching controller (including $x^*$) can be guaranteed to aggregate a system of $n > 3$ robots starting from a deadlocked configuration, implying that Conjecture~\ref{conj:naggregate} does not hold in general.
Moreover, not all deadlocked configurations are disconnected: \figtext~\ref{fig:deadlockring} shows a connected configuration that will never make progress towards a more compact configuration because all robots are mutually blocked by their neighbors.
Notably, these deadlocks are not observed in practice due to inherent noise in the physical e-puck robots.
Real physics work to aggregation's advantage: if the robots were to ever get ``stuck'' in a deadlock configuration, collisions and slipping perturb the precise balancing of forces to allow the robots to push past one another.
This motivates an explicit inclusion and modeling of noise in the algorithm, which we will return to in Sections~\ref{sec:robust} and~\ref{sec:discrete}.

\begin{figure}
    \centering
    \includegraphics[scale=0.28]{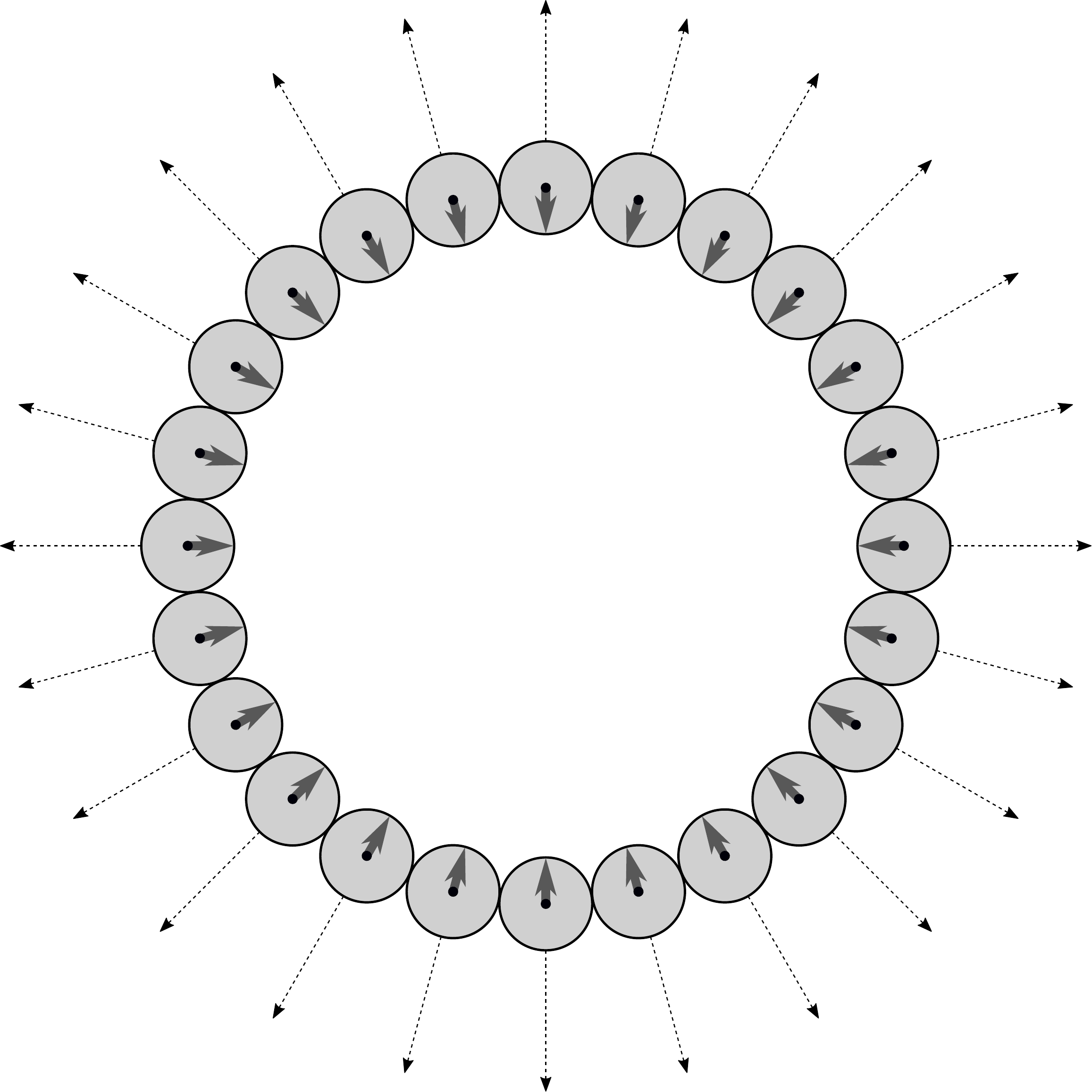}
    \caption{A connected deadlocked configuration that would remain non-compact indefinitely.}
    \label{fig:deadlockring}
\end{figure}

Aaron Becker had conjectured at Dagstuhl Seminar 18331~\cite{Berman2019-dagstuhlprogmat} that symmetry could also lead to livelock, a second type of negative result for the Gauci et al.\ algorithm.
In particular, Becker conjectured that robots initially organized in a cycle (e.g., \figtext~\ref{fig:symmetry:config} for $n = 3$) would traverse a ``symmetric dance'' in perpetuity without converging to an aggregated state when using controller $x^*$.
However, simulations do not support this conjecture.
\figtext~\ref{fig:symmetry:evo} shows that while swarms of various sizes initialized in the symmetric cycle configuration do exhibit an oscillatory behavior, they always reach and remain near the minimum dispersion value indicating near-optimal aggregation.
Interestingly, these unique initial conditions cause small swarms to reach and remain in an oscillatory cycle where they touch and move apart infinitely often.
Larger swarms break symmetry through collisions once the robots touch.

\begin{figure}
    \centering
    \begin{subfigure}{.4\textwidth}
        \centering
        \includegraphics[width=\textwidth]{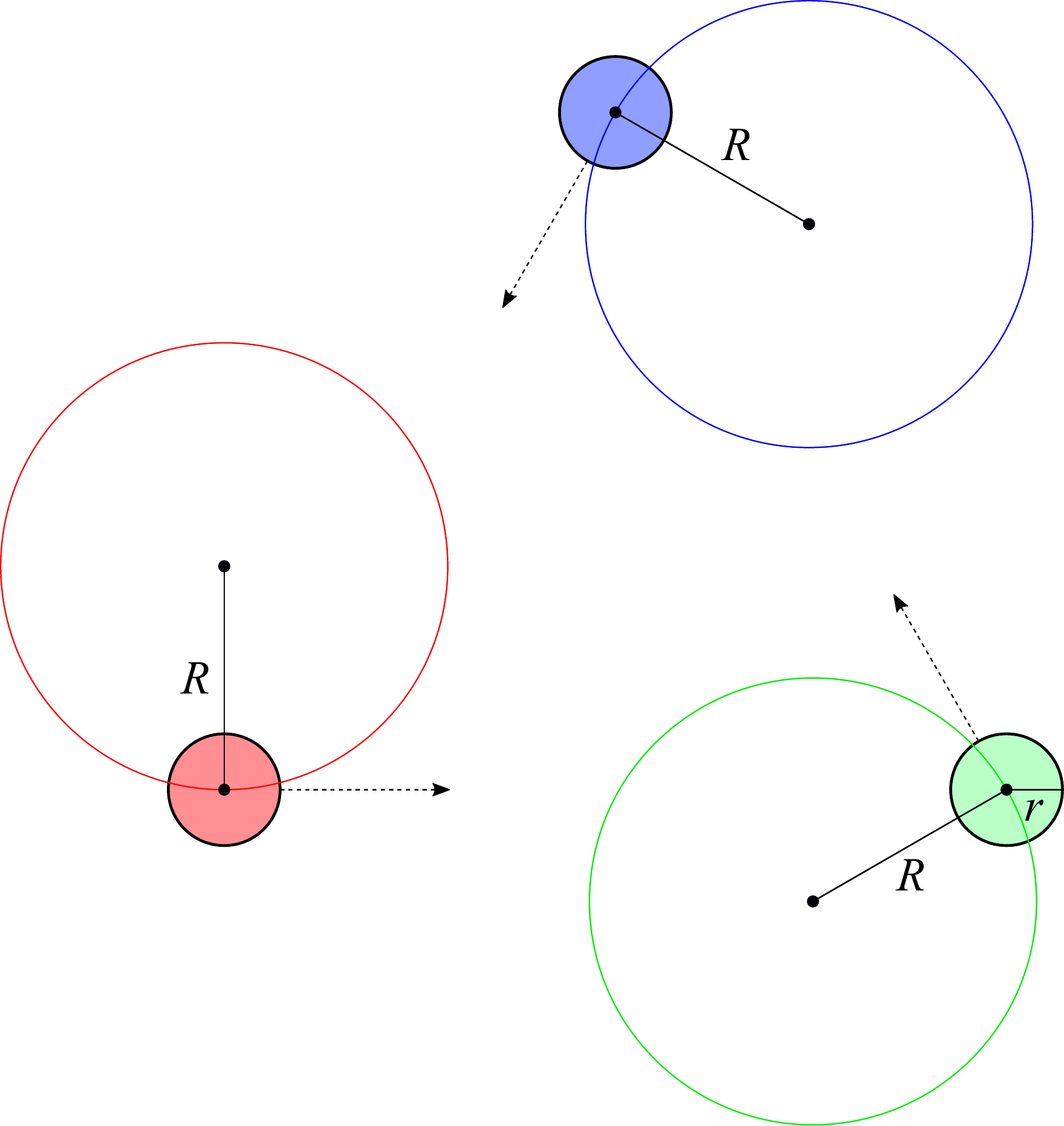}
        \caption{\centering}
        \label{fig:symmetry:config}
    \end{subfigure}
    \hfill
    \begin{subfigure}{.58\textwidth}
        \centering
        \includegraphics[width=\textwidth]{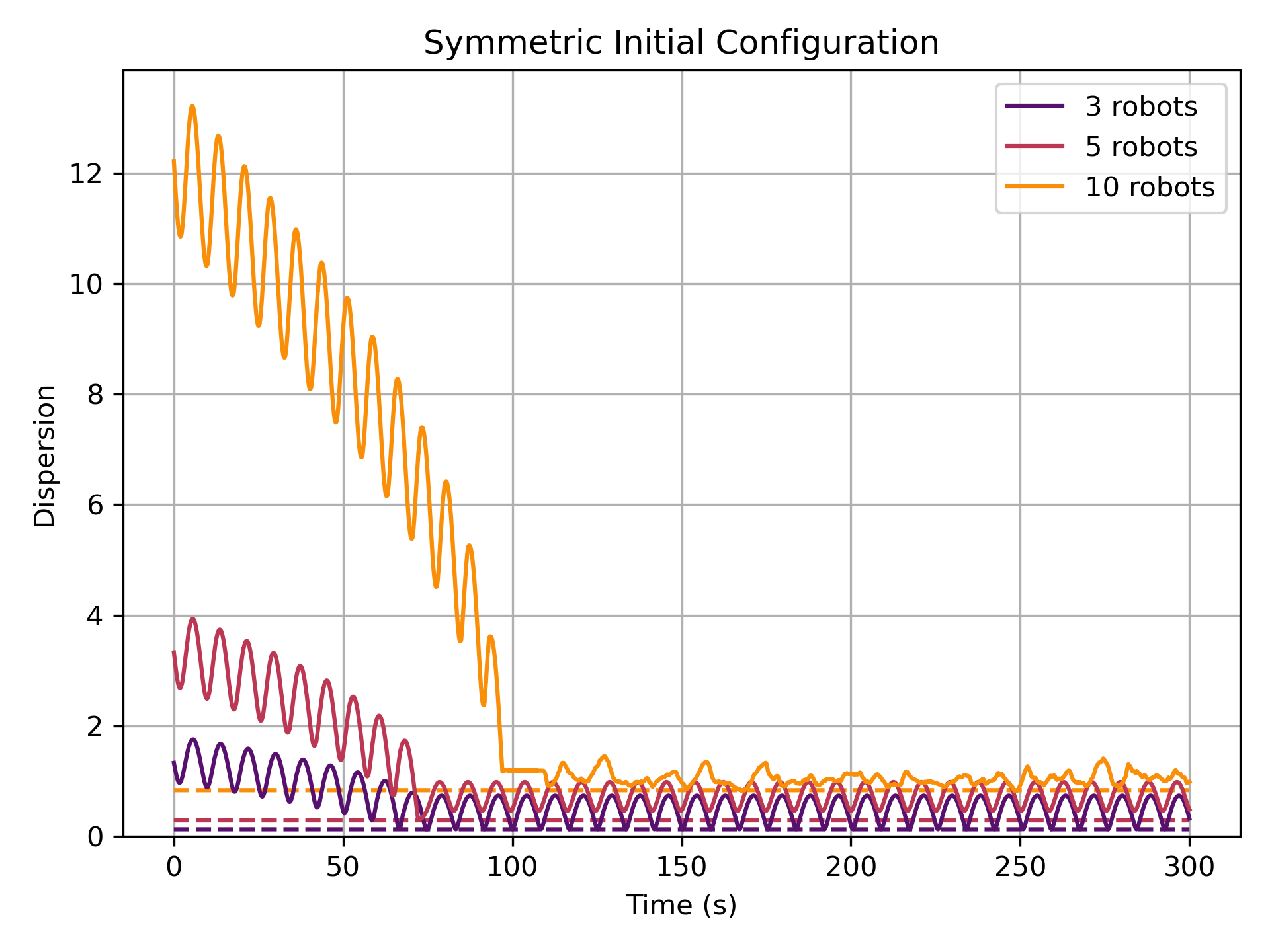}
        \caption{\centering}
        \label{fig:symmetry:evo}
    \end{subfigure}
    \caption{\textbf{\textsf{(a)}} An example symmetric configuration of $n = 3$ robots that was conjectured to produce livelock.
    Reproduced from a private communication from Aaron Becker.
    \textbf{\textsf{(b)}} Dispersion over time for swarms of $n = 3$ (purple), $n = 5$ (magenta), and $n = 10$ (orange) robots with symmetric initial configurations analogous to that of \figtext~\ref{fig:symmetry:config}.
    Dashed lines show the theoretical minimum dispersion value (hexagonal packing) for the given system size.}
    \label{fig:symmetry}
\end{figure}

\section{Robustness to Error and Noise} \label{sec:robust}

Motivated by the role of collisions and perturbations in freeing swarms from potential deadlocks, we next investigate the algorithm's \textit{robustness} to varying magnitudes of error and noise.
Our simulation platform models robots as circular rigid bodies in two dimensions, capturing all translation, rotation, and collision forces acting on the robots.
Forces are combined and integrated iteratively over $5$ ms time steps to obtain the translation and rotation of each robot.
\figtext~\ref{fig:timeevols} shows each of the four aggregation metrics for a baseline run on a swarm of $n = 100$ robots with no explicitly added noise.
All four metrics demonstrate the system’s steady but non-monotonic progress towards aggregation.
Smallest enclosing disc circumference, convex hull perimeter, and dispersion show qualitatively similar progressions while the cluster fraction highlights when individual connected components join together.

\begin{figure}
    \centering
    \begin{subfigure}{.49\textwidth}
        \centering
        \includegraphics[width=\textwidth]{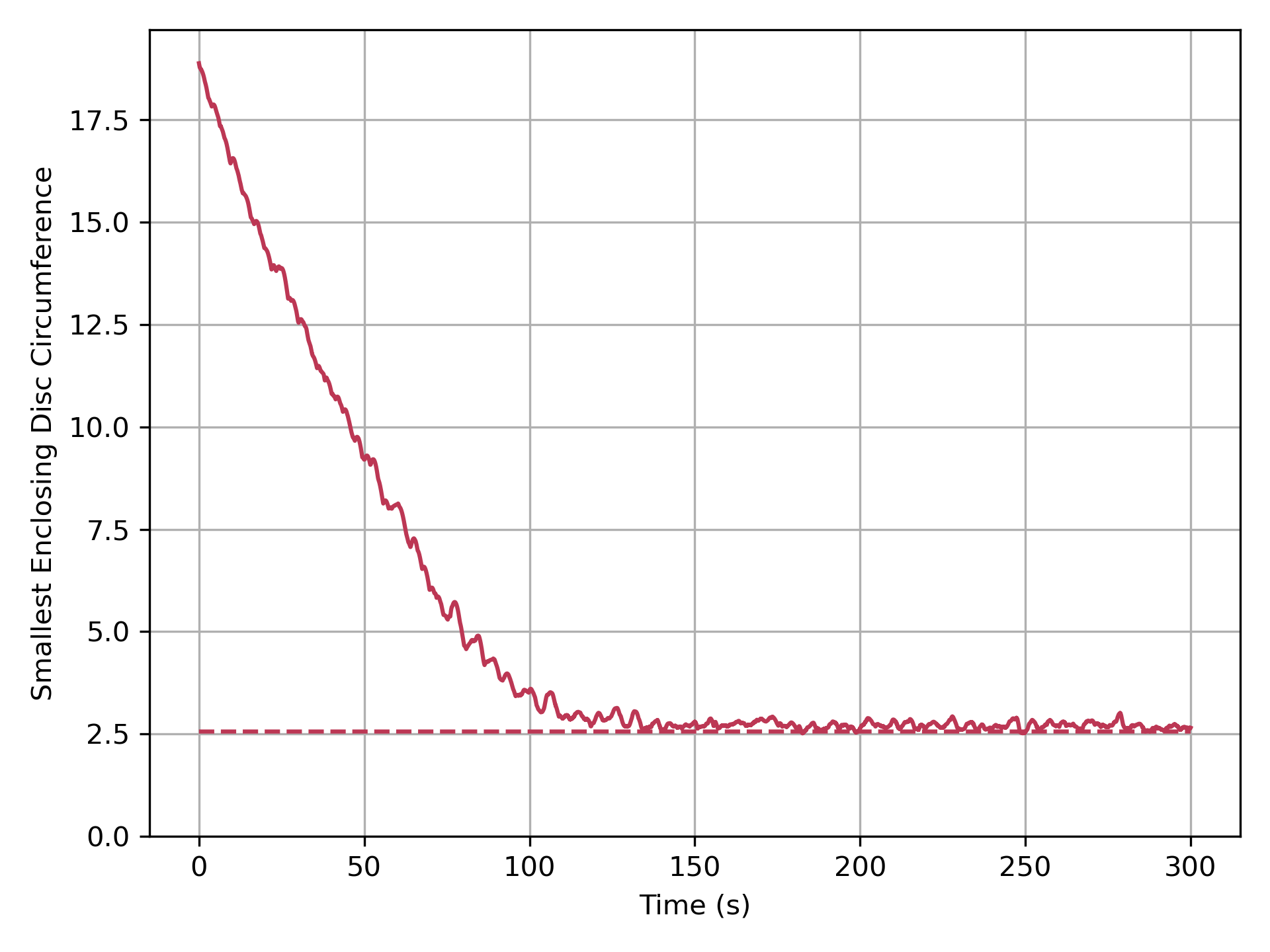}
        \caption{\centering Smallest Enclosing Disc Circumference}
        \label{fig:timeevols:disc}
    \end{subfigure}
    \hfill
    \begin{subfigure}{.49\textwidth}
        \centering
        \includegraphics[width=\textwidth]{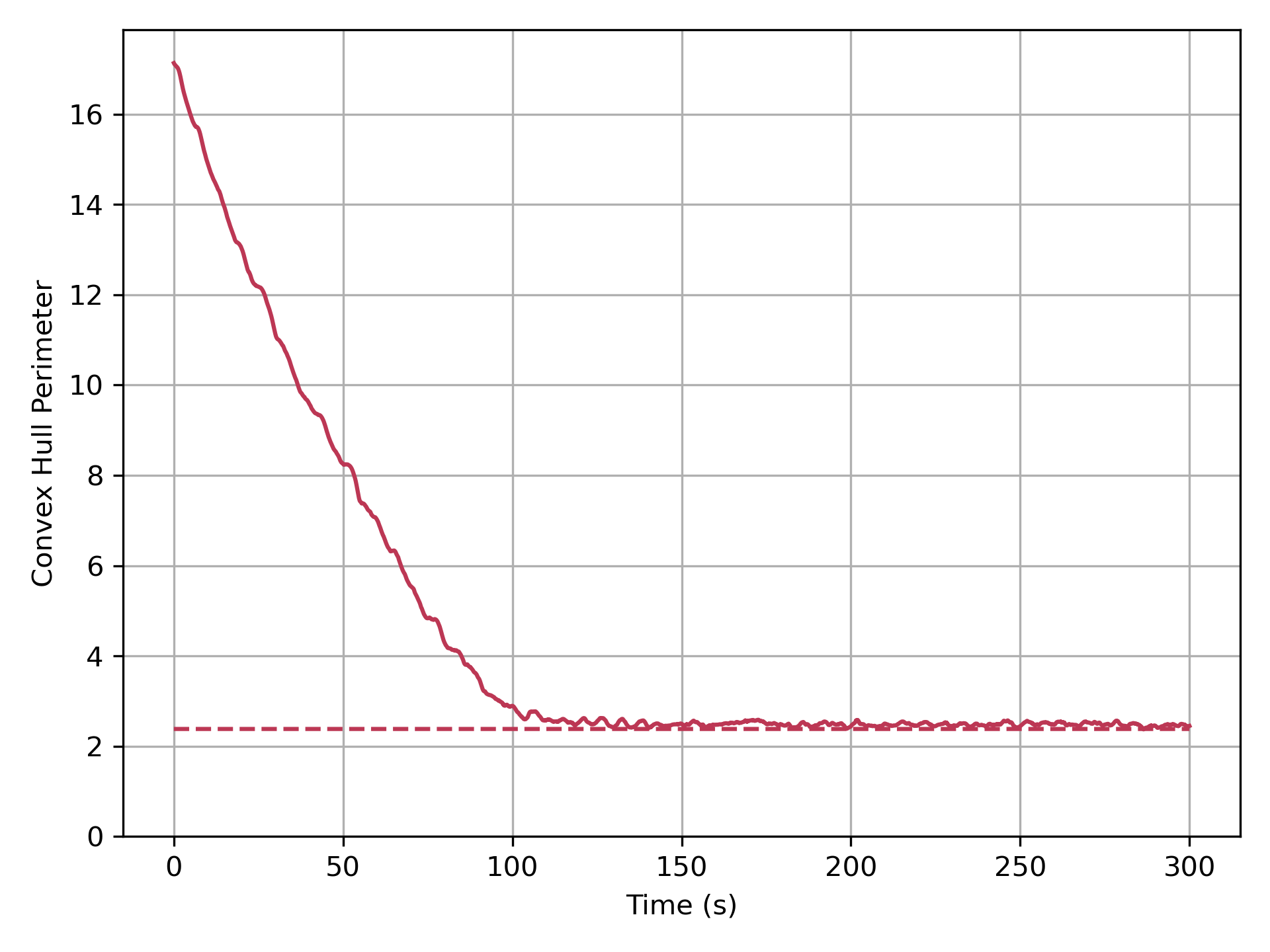}
        \caption{\centering Convex Hull Perimeter}
        \label{fig:timeevols:convex}
    \end{subfigure}\\ \medskip
    \begin{subfigure}{.49\textwidth}
        \centering
        \includegraphics[width=\textwidth]{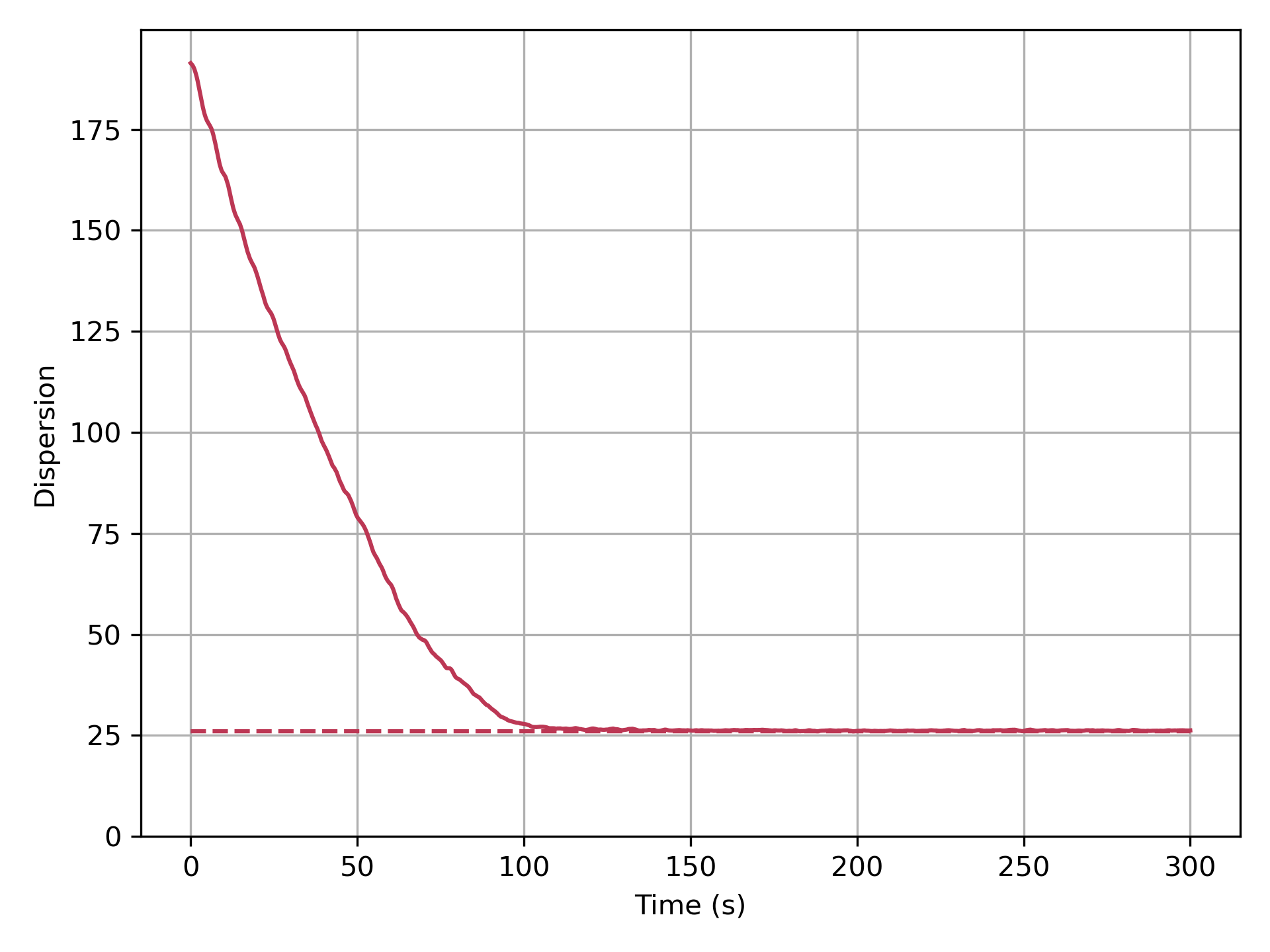}
        \caption{\centering Dispersion}
        \label{fig:timeevols:dispersion}
    \end{subfigure}
    \hfill
    \begin{subfigure}{.49\textwidth}
        \centering
        \includegraphics[width=\textwidth]{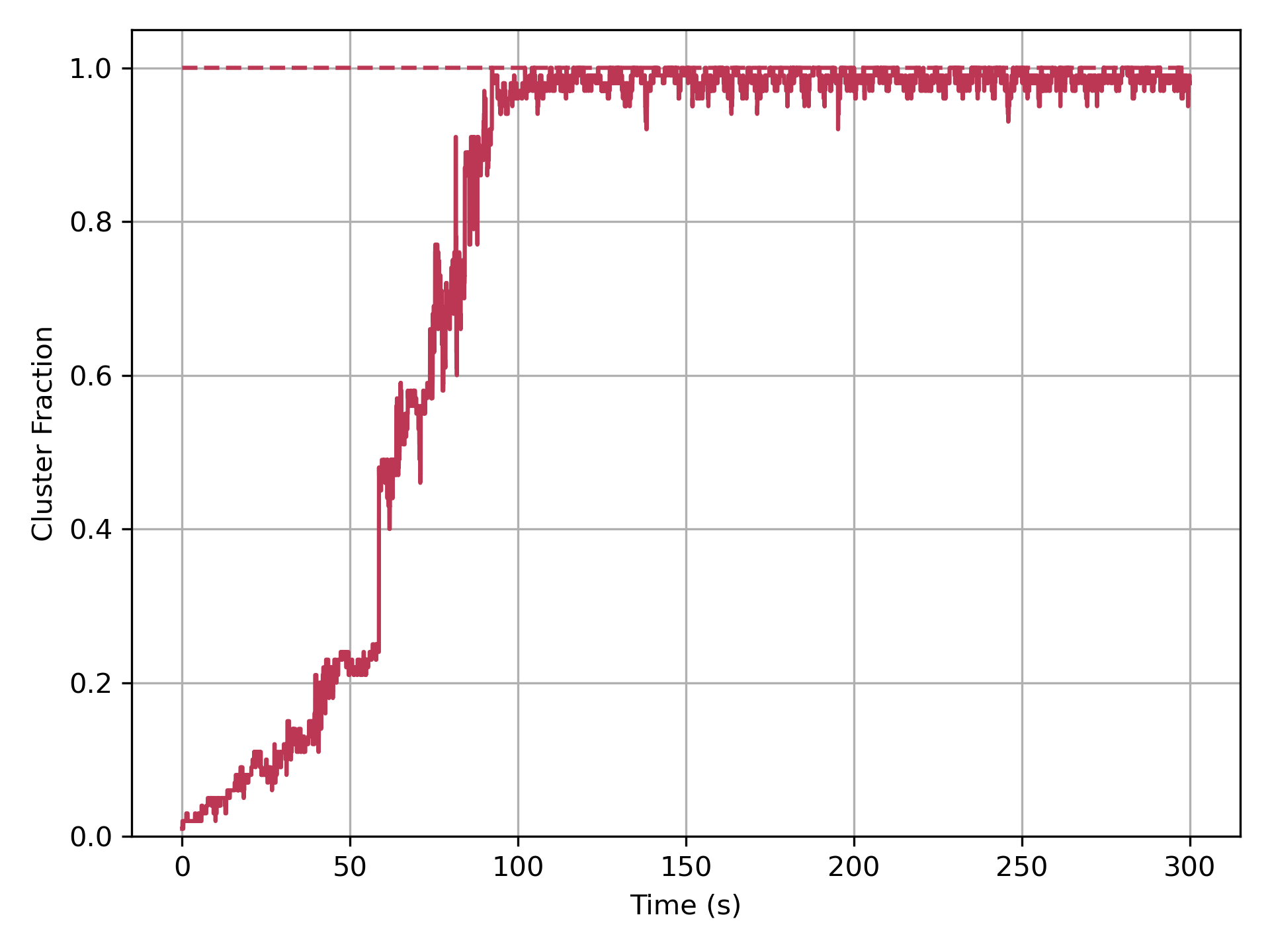}
        \caption{\centering Cluster Fraction}
        \label{fig:timeevols:cluster}
    \end{subfigure}
    \caption{Time evolutions of the four aggregation metrics for the same execution of $x^*$ by a system of $n = 100$ robots for 300 seconds with no explicitly added noise.
    Dashed lines indicate the optimal value for each aggregation metric given the number of robots $n$.}
    \label{fig:timeevols}
\end{figure}

We study the effects of two different forms of noise: \textit{motion noise} and \textit{error probability}.
For motion noise, each robot at each time step experiences an applied force of a random magnitude in $[0, m^*]$ in a random direction.
The parameter $m^*$ defines the maximum noise force (in newtons) that can be applied to a robot in a single time step.
For error probability, each robot has the same probability $p \in [0,1]$ of receiving the incorrect feedback from its sight sensor at each time step; more formally, a robot will receive the correct feedback $I$ with probability $(1-p)$ and the opposite, incorrect feedback $1 - I$ with probability $p$.\footnote{Our formulation of an ``error probability'' $p$ is equivalent to ``sensory noise'' in~\cite{Gauci2014-aggregation} when the false positive and false negative probabilities are both equal to $p$.}
The robot then proceeds with the algorithm as usual based on the feedback it receives.

\begin{figure}
    \centering
    \begin{subfigure}{.49\textwidth}
        \centering
        \includegraphics[width=\textwidth]{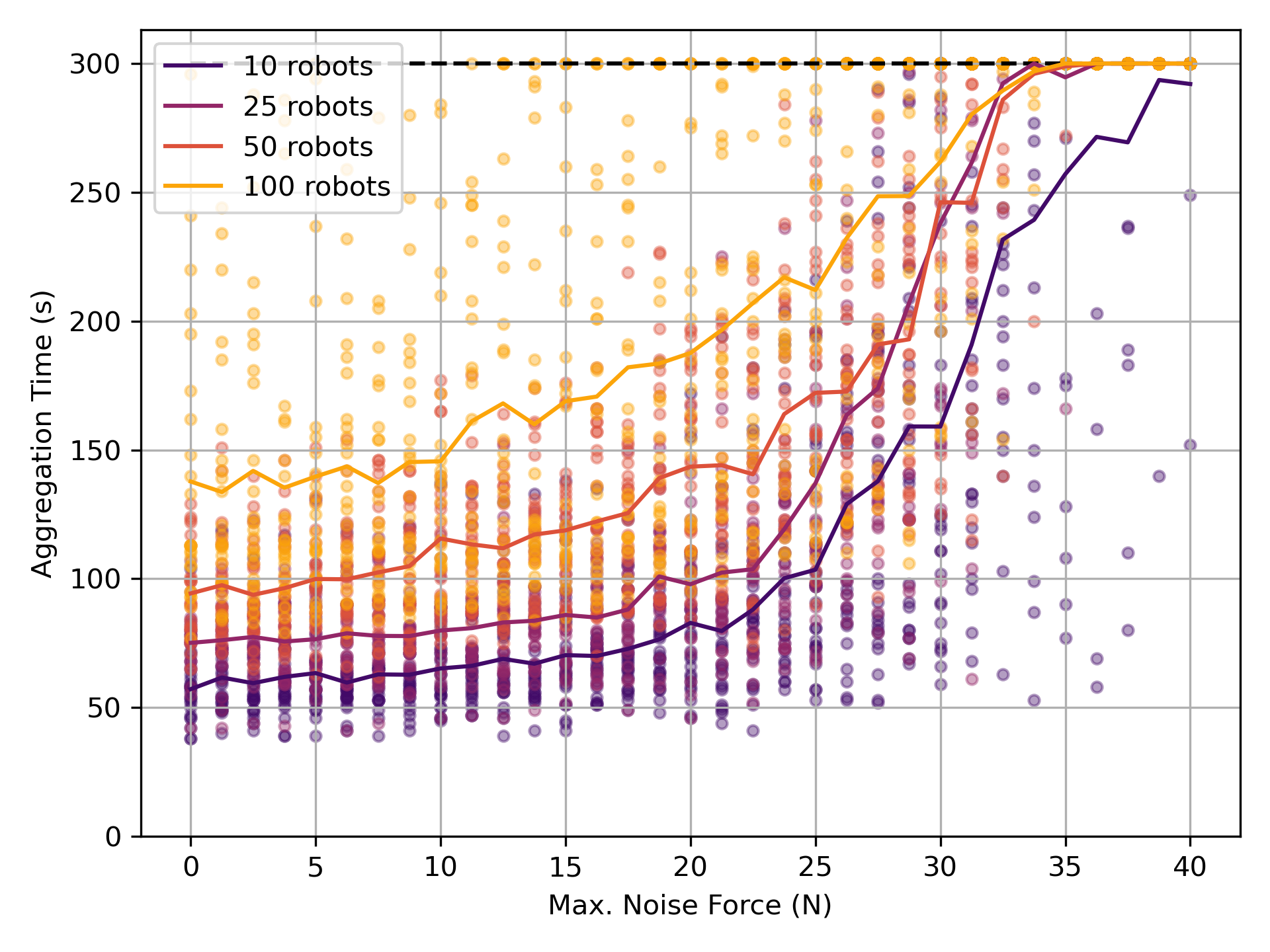}
        \caption{\centering Motion Noise}
        \label{fig:noise-motion}
    \end{subfigure}
    \hfill
    \begin{subfigure}{.49\textwidth}
        \centering
        \includegraphics[width=\textwidth]{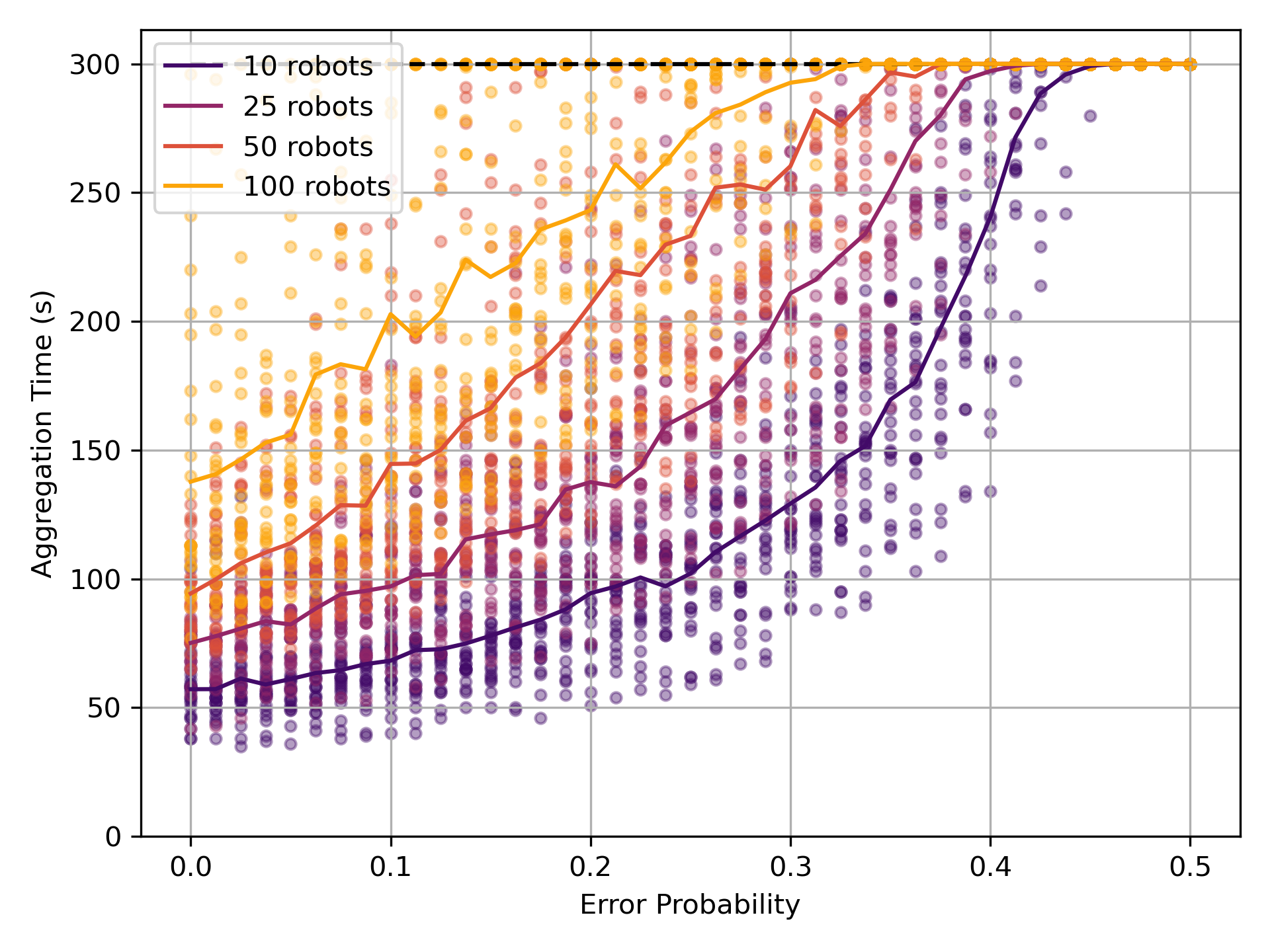}
        \caption{\centering Error Probability}
        \label{fig:noise-errprob}
    \end{subfigure}
    \caption{The time required to reach aggregation for different magnitudes of \textbf{\textsf{(a)}} motion noise and \textbf{\textsf{(b)}} error probability for systems of $n = 10$ (purple), $n = 25$ (magenta), $n = 50$ (red), and $n = 100$ (orange) robots.
    Each experiment for a given system size and noise strength was repeated 25 times (scatter plot); average runtimes are shown as solid lines.
    We consider systems that are within $15\%$ of the minimum dispersion value as aggregated.
    The dashed line at $300$ seconds indicates the simulation runtime cutoff at which the run is determined to be non-aggregating.}
    \label{fig:noise-contin}
\end{figure}

In general, as the magnitude of error increases, so does the time required to achieve aggregation.
The algorithm exhibits robustness to low magnitudes of motion noise with the average time to aggregation remaining relatively steady for $m^* \leq 5$ N and increasing only minimally for $5 \leq m^* \leq 20$ N (\figtext~\ref{fig:noise-motion}).
With larger magnitudes of motion noise ($m^* > 20$ N), average time to aggregation increases significantly, with many runs reaching the limit for simulation time before aggregation is reached.
A similar trend is evident for error probability (\figtext~\ref{fig:noise-errprob}).
The algorithm exhibits robustness for small error probabilities $p \in [0, 0.05]$ with the average time to aggregation rising steadily with increased error until nearly all runs reach the simulation time limit.
Intuitively, while small amounts of noise can help the algorithm overcome deadlock without degrading performance, too much noise interferes significantly with the algorithm's ability to progress towards aggregation.

\section{Using a Cone-of-Sight Sensor} \label{sec:cone}

We next analyze a generalization of the algorithm where each robot has a \textit{cone-of-sight sensor} of size $\beta$ instead of a line-of-sight sensor ($\beta = 0$).
This was left as future work in~\cite{Gauci2014-aggregation} and was briefly considered in Gauci's PhD dissertation~\cite{Gauci2014-phdthesis} where, for each $\beta \in \{0^\circ, 30^\circ, \ldots, 180^\circ\}$, the best performing controller $x_\beta$ was found and compared against the others.
Here we take a complementary approach, studying the performance of the original controller $x^*$ as $\beta$ varies.

We begin by proving that, in the case of one static robot and one robot executing the generalized algorithm, using a cone-of-sight sensor with size $\beta > 0$ can improve the time to aggregation by a linear factor over the original algorithm.
Intuitively, a robot with a cone-of-sight sensor will rotate in place longer than its line-of-sight counterpart, consequently moving its center of rotation a greater distance towards the other robot with each revolution.

\begin{figure}
    \centering
    \includegraphics[width=0.52\textwidth]{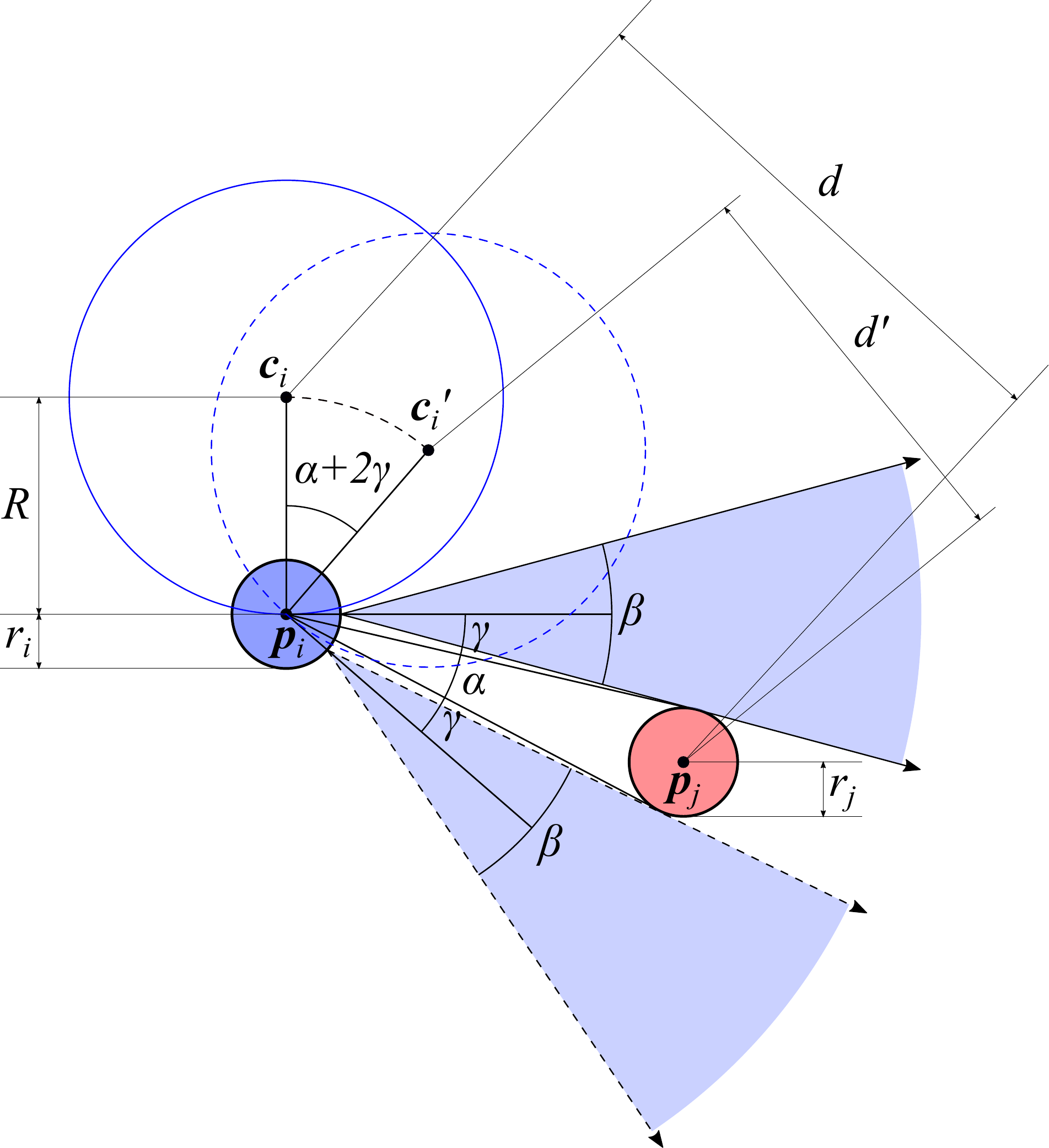}
    \caption{The setup considered in the proof of Theorem~\ref{thm:coneofsight}.
    Robot $i$ is moving and has a cone-of-sight sensor with size $\beta$ while robot $j$ is static.}
    \label{fig:cone}
\end{figure}

\begin{theorem} \label{thm:coneofsight}
    One moving robot using a cone-of-sight sensor of size $\beta \in (0, \pi)$ will always aggregate with another static robot in
    \[m < \left\lceil\frac{(d_0 - R - r_i - r_j)(R + 2r_i)}{2\sqrt{3}Rr_i\sin((1 - 1/\sqrt{3})\cdot\beta/2)}\right\rceil\]
    rotations around its center of rotation, where $d_0$ is the initial value of $||\boldsymbol{p}_j - \boldsymbol{c}_i||$.\footnote{The initial ideas behind this proof were formulated at Dagstuhl Seminar 18331~\cite{Berman2019-dagstuhlprogmat} in collaboration with Roderich Gro\ss\ and Aaron Becker.}
\end{theorem}
\begin{proof}
    Consider a robot $i$ executing the generalized algorithm at position $\boldsymbol{p}_i$ with center of rotation $\boldsymbol{c}_i$ and a static robot $j$ at position $\boldsymbol{p}_j$.
    As in the proofs of Theorems 5.1 and 5.2 in~\cite{Gauci2014-aggregation}, we first consider the scenario shown in \figtext~\ref{fig:cone} and derive an expression for $d' = ||\boldsymbol{p}_j - \boldsymbol{c}_i'||$ in terms of $d = ||\boldsymbol{p}_j - \boldsymbol{c}_i||$.
    W.l.o.g., let $\boldsymbol{c}_i = [0, 0]^T$ and let the axis of the cone-of-sight sensor of robot $i$ point horizontally right at the moment it starts seeing robot $j$.
    Then the position of robot $j$ is given by
    \[\boldsymbol{p}_j = \left[\begin{array}{c}
        \frac{r_j\cos(\alpha/2 + \gamma)}{\sin(\alpha/2)} \\
        -\left(R + \frac{r_j\sin(\alpha/2 + \gamma)}{\sin(\alpha/2)}\right)
    \end{array}\right].\]
    Substituting this position into the distance $d = ||\boldsymbol{p}_j - \boldsymbol{c}_i||$ yields
    \[d^2 = \left(\frac{r_j\cos(\alpha/2 + \gamma)}{\sin(\alpha/2)}\right)^2 + \left(R + \frac{r_j\sin(\alpha/2 + \gamma)}{\sin(\alpha/2)}\right)^2 = R^2 + \frac{2Rr_j\sin(\alpha/2 + \gamma)}{\sin(\alpha/2)} + \frac{r_j^2}{\sin^2(\alpha/2)}.\]
    Using a line-of-sight sensor, robot $i$ would only rotate $\alpha$ before it no longer sees robot $j$; however, with a cone-of-sight sensor of size $\beta$, robot $i$ rotates $\alpha + 2\gamma$ before robot $j$ leaves its sight, where $\gamma$ is the angle from the cone-of-sight axis to the first line intersecting $\boldsymbol{p}_i$ that is tangent to robot $j$.
    With this cone-of-sight sensor, $\boldsymbol{c}_i'$ is given by
    \[\boldsymbol{c}_i' = \left[\begin{array}{c}
        R\sin(\alpha + 2\gamma) \\
        R(\cos(\alpha + 2\gamma) - 1)
    \end{array}\right].\]
    Substituting this new center of rotation into the distance $d' = ||\boldsymbol{p}_j - \boldsymbol{c}_i'||$ yields
    \begin{align*}
        d' &= \sqrt{
        \begin{aligned}
        &\left(\frac{r_j\cos(\alpha/2 + \gamma)}{\sin(\alpha/2)} - R\sin(\alpha + 2\gamma)\right)^2 \\
        &+ \left(-\left(R + \frac{r_j\sin(\alpha/2 + \gamma)}{\sin(\alpha/2)}\right) - R(\cos(\alpha + 2\gamma) - 1)\right)^2
        \end{aligned}
        } \\
        &= \sqrt{
        \begin{aligned}
        &\frac{r_j^2\cos^2(\alpha/2 + \gamma)}{\sin^2(\alpha/2)} - \frac{2Rr_j\cos(\alpha/2 + \gamma)\sin(\alpha + 2\gamma)}{\sin(\alpha/2)} + R^2\sin^2(\alpha + 2\gamma) \\
        &+ \frac{r_j^2\sin^2(\alpha/2 + \gamma)}{\sin^2(\alpha/2)} + \frac{2Rr_j\sin(\alpha/2 + \gamma)\cos(\alpha + 2\gamma)}{\sin(\alpha/2)} + R^2\cos^2(\alpha + 2\gamma)
        \end{aligned}
        } \\
        &= \sqrt{R^2 + \frac{r_j^2}{\sin^2(\alpha/2)} + \frac{2Rr_j(\sin(\alpha/2 + \gamma)\cos(\alpha + 2\gamma) - \cos(\alpha/2 + \gamma)\sin(\alpha + 2\gamma))}{\sin(\alpha/2)}} \\
        &= \sqrt{d^2 + \frac{2Rr_j(\sin(\alpha/2 + \gamma - (\alpha + 2\gamma)) - \sin(\alpha/2 + \gamma))}{\sin(\alpha/2)}} \\
        &= \sqrt{d^2 - \frac{4Rr_j\sin\left(\alpha/2 + \gamma\right)}{\sin(\alpha/2)}}.
    \end{align*}
    Note that this relation contains the result proven in Theorem 5.1 of~\cite{Gauci2014-aggregation} as a special case by setting $\beta = 0$ (and thus $\gamma = 0$), which corresponds to a line-of-sight sensor.
    To bound the number of $d \to d'$ updates required until $d \leq R + r_i + r_j$ (i.e., until the robots have aggregated), we write the following recurrence relation, where $\hat{d}_m = d_m^2$ and $\hat{d}_m > (R + r_i + r_j)^2$:
    \[\hat{d}_{m+1} = \hat{d}_{m} - \frac{4Rr_j\sin\left(\alpha/2 + \gamma\right)}{\sin(\alpha/2)}.\]
    Observe that $\alpha$ is the largest when the two robots are touching, and --- assuming $r_i = r_j$, i.e., the two robots are the same size --- it is easy to see that $\alpha \leq \pi/3$.
    Also, $\gamma$ is at least $0$ and at most $\beta / 2$; thus, by supposition, $\gamma < \pi/2$.
    Thus, by the angle sum identity,
    \[\hat{d}_{m+1} < \hat{d}_m - \frac{4Rr_j\cos(\alpha/2)\sin(\gamma)}{\sin(\alpha/2)}.\]
    Again, since $\alpha \leq \pi/3$, we have $\cos(\alpha / 2) \geq \sqrt{3}/2$.
    By inspection, we also have $\sin(\alpha/2) < r_j/(d_m - R)$, yielding
    \[\hat{d}_{m+1} < \hat{d}_{m} - 2\sqrt{3}R(d_m - R)\sin(\gamma).\]
    
    Let $k_1 > 1$ be a constant such that $r_i = r_j = R/k_1$, which must exist since $r_i$, $r_j$, and $R$ are constants and $r_i = r_j < R$.
    Since the robots have not yet aggregated, we have $d_m > R + r_i + r_j = (1 + 2/k_1)R$.
    We use this to show
    \[R < \frac{d_m}{1 + 2/k_1} = k_2d_m,\]
    where $k_2 = 1/(1 + 2/k_1) < 1$ is a constant.
    Returning to our recurrence relation:
    \[\hat{d}_{m+1} < \hat{d}_m - 2\sqrt{3}R(d_m - k_2d_m)\sin(\gamma) = \hat{d}_m - 2\sqrt{3}Rd_m(1 - k_2)\sin(\gamma)\]
    Recalling that $\hat{d}_m = d_m^2$, we have
    \[d_{m+1} < \sqrt{\hat{d}_m - 2\sqrt{3}Rd_m(1 - k_2)\sin(\gamma)} < d_m - \sqrt{3}R(1 - k_2)\sin(\gamma),\]
    where the second inequality can be verified by squaring both sides as follows:
    \begin{align*}
        \hat{d}_m - 2\sqrt{3}Rd_m(1 - k_2)\sin(\gamma) &< \hat{d}_{m} - 2\sqrt{3}Rd_m(1 - k_2)\sin(\gamma) + 3R^2(1-k_2)^2\sin^2(\gamma) \\
        0 &< 3R^2(1-k_2)^2\sin^2(\gamma),
    \end{align*}
    which holds because $R > 0$, $1-k_2 > 0$, and $\gamma \in (0, \pi/2)$.
    
    As a final upper bound on $d_{m+1}$, we lower bound the angle $\gamma$ as a function of the constant size of the cone-of-sight sensor $\beta$ as $\gamma \geq (1 - 1/\sqrt{3})\cdot\beta/2$ (see Appendix~\ref{app:lowerbound}), yielding
    \[d_{m+1} < d_m - \sqrt{3}R(1 - k_2)\sin((1 - 1/\sqrt{3})\cdot\beta/2)\]
    This yields the solution
    \[d_m < d_0 - m\sqrt{3}R(1 - k_2)\sin((1 - 1/\sqrt{3})\cdot\beta/2), \quad d_m > R + r_i + r_j\]
    The number of $d \to d'$ updates required until $d \leq R + r_i + r_j$ is now given by setting $d_m = R + r_i + r_j$ in this solution and solving for $m$, which yields
    \[m < \left\lceil\frac{d_0 - R - r_i - r_j}{\sqrt{3}R(1 - k_2)\sin((1 - 1/\sqrt{3})\cdot\beta/2)}\right\rceil
    = \left\lceil\frac{(d_0 - R - r_i - r_j)(R + 2r_i)}{2\sqrt{3}Rr_i\sin((1 - 1/\sqrt{3})\cdot\beta/2)}\right\rceil.\]
    We note that this bound on the number of required updates $m$ has a linear dependence on $d_0$ while the original bound proven in Theorem 5.2 of~\cite{Gauci2014-aggregation} for line-of-sight sensors depended on $d_0^2$, demonstrating a linear speedup with cone-of-sight sensors.
\end{proof}

\begin{figure}
    \centering
    \includegraphics[width=.65\textwidth]{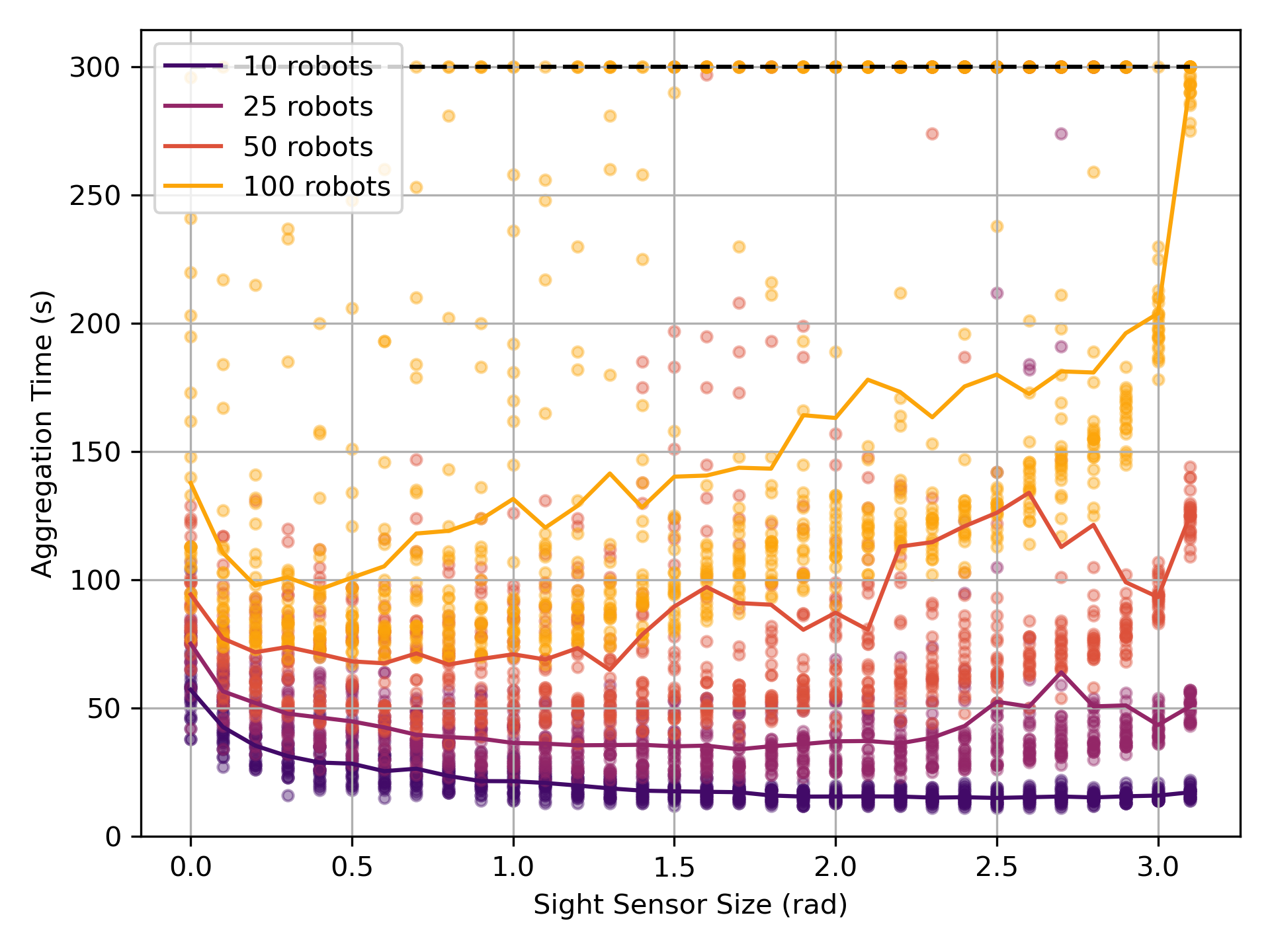}
    \caption{The effects of cone-of-sight sensor size on the algorithm's time to aggregation for systems of $n = 10$ (purple), $n = 25$ (magenta), $n = 50$ (red), and $n = 100$ (orange) robots.
    Each experiment for a given system size and sensor size was repeated 25 times (scatter plot); average runtimes are shown as solid lines.
    We consider systems that are within $15\%$ of the minimum dispersion value as aggregated.
    The dashed line at $300$ seconds indicates the simulation runtime cutoff at which the run is determined to be non-aggregating.}
    \label{fig:contincone}
\end{figure}

This theorem establishes a provable linear speedup for systems of $n = 2$ robots with cone-of-sight sensors over those with line-of-sight sensors.
However, simulation results show that as the number of robots increases, the speedup from using cone-of-sight sensors diminishes (\figtext~\ref{fig:contincone}).
All systems benefit from small cone-of-sight sensors --- i.e., $\beta \in (0, 0.5)$ --- reaching aggregation in significantly less time.
With larger systems, however, large cone-of-sight sensors can be detrimental as robots see others more often than not, causing them to primarily rotate in place without making progress towards aggregation.
This highlights a delicate balance between the algorithm's two modes (rotating around the center of rotation and rotating in place) with $\beta$ indirectly affecting how much time is spent in each.

\section{A Noisy, Discrete Adaptation} \label{sec:discrete}

Recall that our original goal was to prove Conjecture~\ref{conj:naggregate}: that algorithm $x^*$ achieves aggregation for all system sizes $n > 2$.
We proved in Theorem~\ref{thm:deadlock} that this is not always the case due to deadlock configurations.
To investigate under what conditions aggregation can be achieved reliably, we developed a \textit{discrete adaptation} of controller $x^*$ that simplifies the algorithm to its core movement rules.
Because the discrete setting lacks any natural physics, this discrete adaptation requires an \textit{explicit inclusion of noise} in order to escape deadlock.

\begin{figure}
    \centering
    \begin{subfigure}{.32\textwidth}
        \centering
        \includegraphics[width=\textwidth]{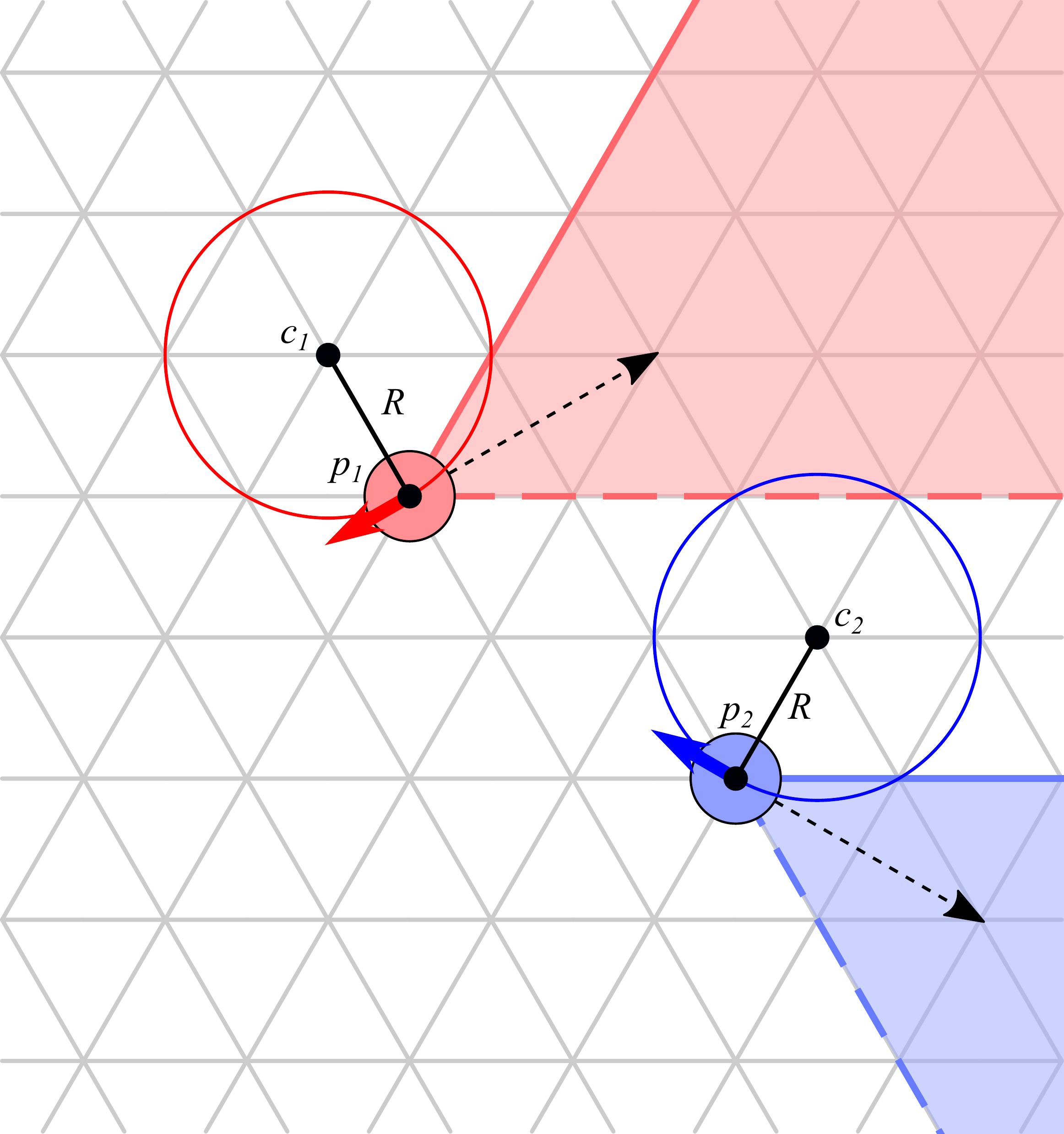}
        \caption{\centering}
        \label{fig:discrete:a}
    \end{subfigure}
    \hfill
    \begin{subfigure}{.32\textwidth}
        \centering
        \includegraphics[width=\textwidth]{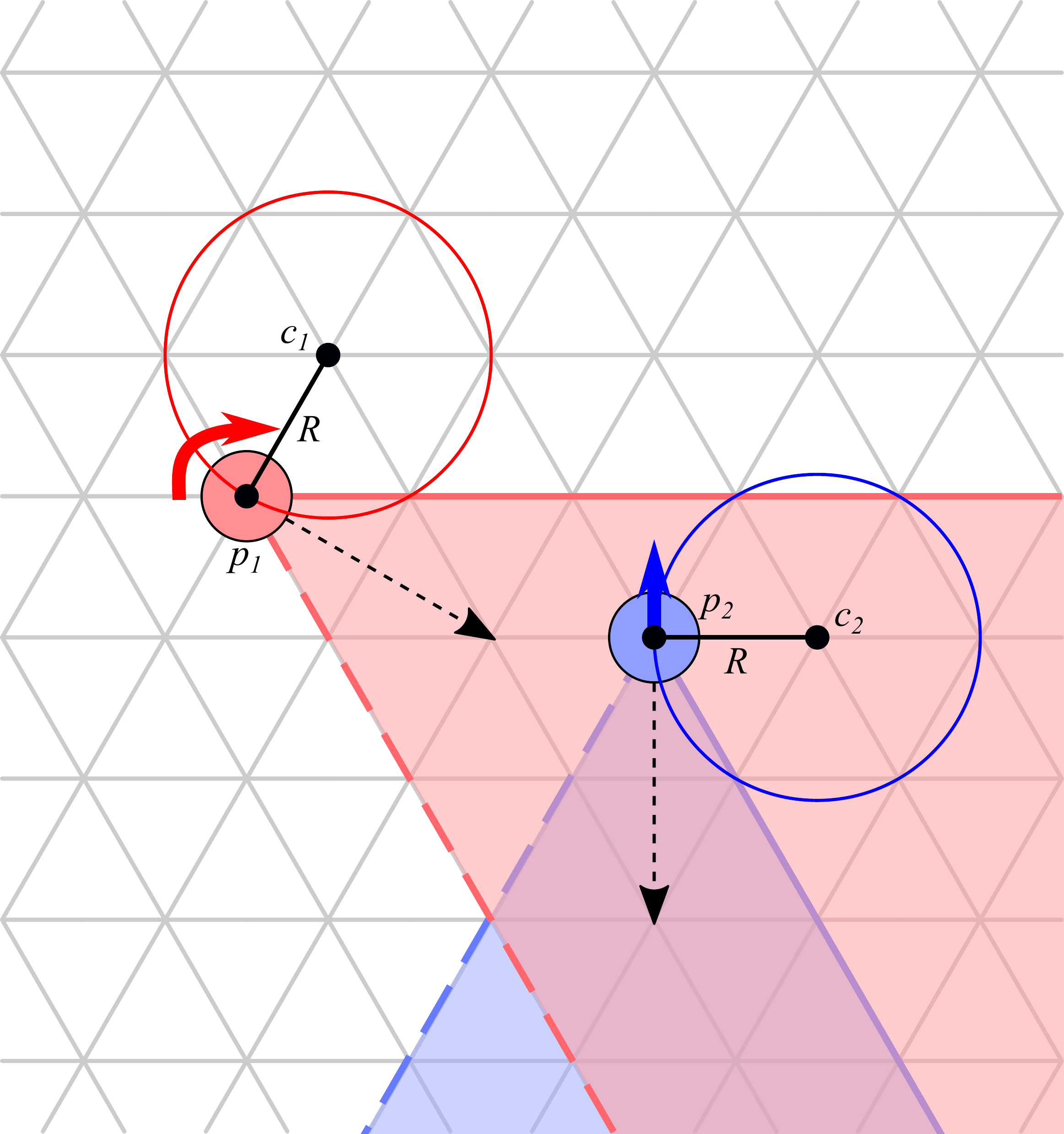}
        \caption{\centering}
        \label{fig:discrete:b}
    \end{subfigure}
    \hfill
    \begin{subfigure}{.32\textwidth}
        \centering
        \includegraphics[width=\textwidth]{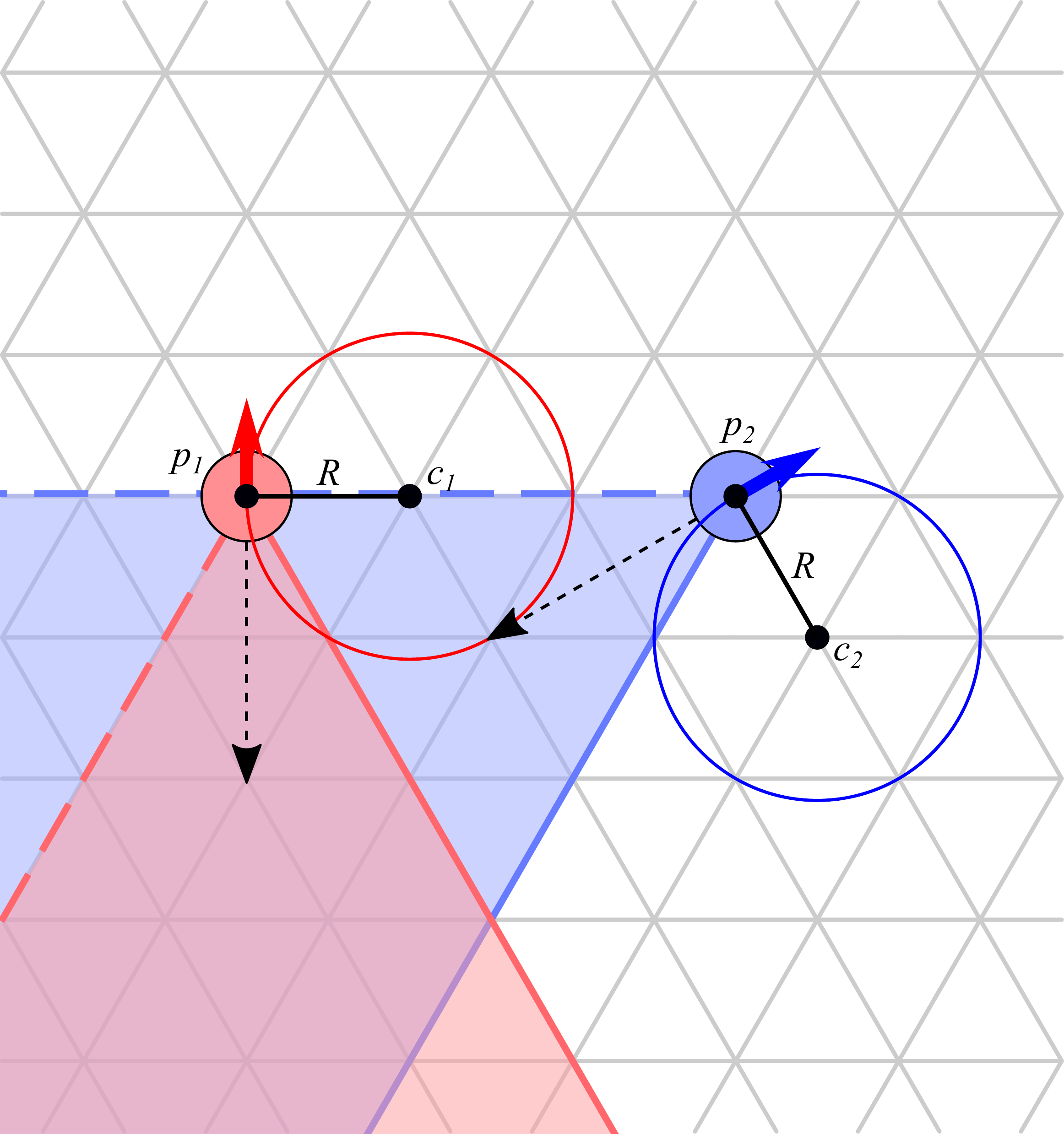}
        \caption{\centering}
        \label{fig:discrete:c}
    \end{subfigure}
    \caption{An example of $n = 2$ robots running the discrete adaptation of $x^*$ (without noise).
    \textbf{\textsf{(a)}} Neither $p_1$ nor $p_2$ see another robot in their cones-of-sight, so they rotate clockwise about their respective centers.
    \textbf{\textsf{(b)}} Robot $p_1$ now sees $p_2$ in its cone-of-sight, so it will rotate clockwise in place; $p_2$ still does not see a robot, so it rotates clockwise about its center.
    \textbf{\textsf{(c)}} Once again, neither robot sees another in their cones-of-sight and will next rotate about their centers. Note that the cones-of-sight are inclusive on their counter-clockwise boundaries sides but exclusive on their clockwise boundaries.}
    \label{fig:discrete}
\end{figure}

In the discrete adaptation, we model space using the triangular lattice $\Gtri = (V, E)$ where the nodes $V$ represent the positions a robot can occupy and the edges $E$ represent connections to adjacent positions (\figtext~\ref{fig:discrete}).
Two robots are \textit{connected} if they occupy nodes $u,v \in V$ such that $(u,v) \in E$.
Motivated by the results in Section~\ref{sec:cone} and restricted by the triangular lattice, we generalize the line-of-sight sensors used in the Gauci et al.\ algorithm~\cite{Gauci2014-aggregation} to \textit{cone-of-sight sensors} capable of seeing a $60^\circ$ section of the lattice.
As in the original algorithm, each robot's center of rotation $c$ is $90^\circ$ counter-clockwise from the direction of its sensor.
At each discrete time step, an arbitrary robot is chosen (e.g., uniformly at random) and performs the steps below (see \figtext~\ref{fig:discrete}).
Discrete time is measured in \textit{rounds}, where a round is complete once each robot has been chosen at least once.

\begin{enumerate}
    \item If no robot is seen in the cone-of-sight and the next clockwise node (say $u$) around $c$ is unoccupied, move to $u$ and adjust the cone's direction $60^\circ$ clockwise.
    \item If a robot is seen in the cone-of-sight, rotate clockwise in place by moving $c$ to the next node clockwise around the robot and adjusting the cone's direction $60^\circ$ clockwise.
    \item Otherwise, do nothing.
\end{enumerate}

In order to overcome deadlock, we consider two different implementations of noise: \textit{error probability} and \textit{deadlock perturbation}.
As in Section~\ref{sec:robust}, both forms of noise allow robots to break symmetry and free themselves from deadlock configurations.
Error probability can be adapted directly from its formulation for the continuous setting (Section~\ref{sec:robust}): each robot has the same probability $p \in [0,1]$ of receiving the incorrect feedback from its cone-of-sight sensor at each step and then proceeds with the algorithm as usual based on the feedback it received.
For deadlock perturbation, each robot is assumed to be able to detect when its rotation around $c$ is blocked by another robot.
Each robot keeps a counter $d \in \mathbb{Z}_{\geq 0}$ and a maximum counter value $d^* > 0$.\footnote{We note that the counters used in deadlock perturbation require the robot to keep persistent state and perform conditional logic, both of which are avoided entirely in the original algorithm.}
The counter remains at $d = 0$ until the robot's rotation becomes blocked.
Then, in each consecutive step that the robot is blocked, $d$ is incremented.
When the robot has been blocked for $d^*$ consecutive steps (indicated by $d = d^*$), the counter is reset to $d = 0$ and the robot performs a perturbation in which --- instead of continuing to attempt to rotate about its center --- it rotates once in place.
This perturbation alters the cone-of-sight's direction so that the robot might possibly be able to break free from deadlock.

\begin{figure}
    \begin{subfigure}{.49\textwidth}
        \centering
        \includegraphics[width=\textwidth]{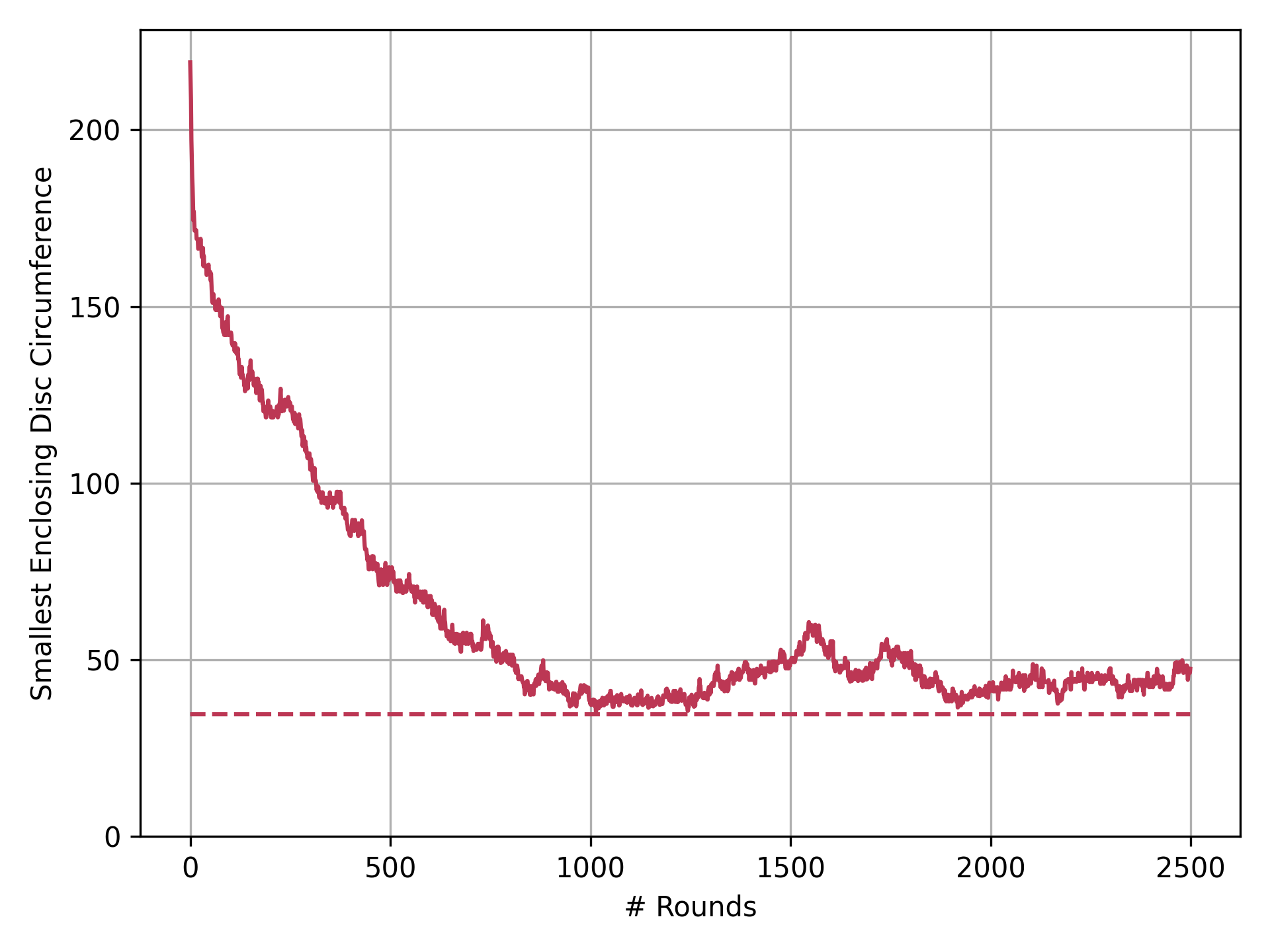}
        \caption{\centering Smallest Enclosing Disc Circumference}
        \label{fig:metrics:disc}
    \end{subfigure}
    \hfill
    \begin{subfigure}{.49\textwidth}
        \centering
        \includegraphics[width=\textwidth]{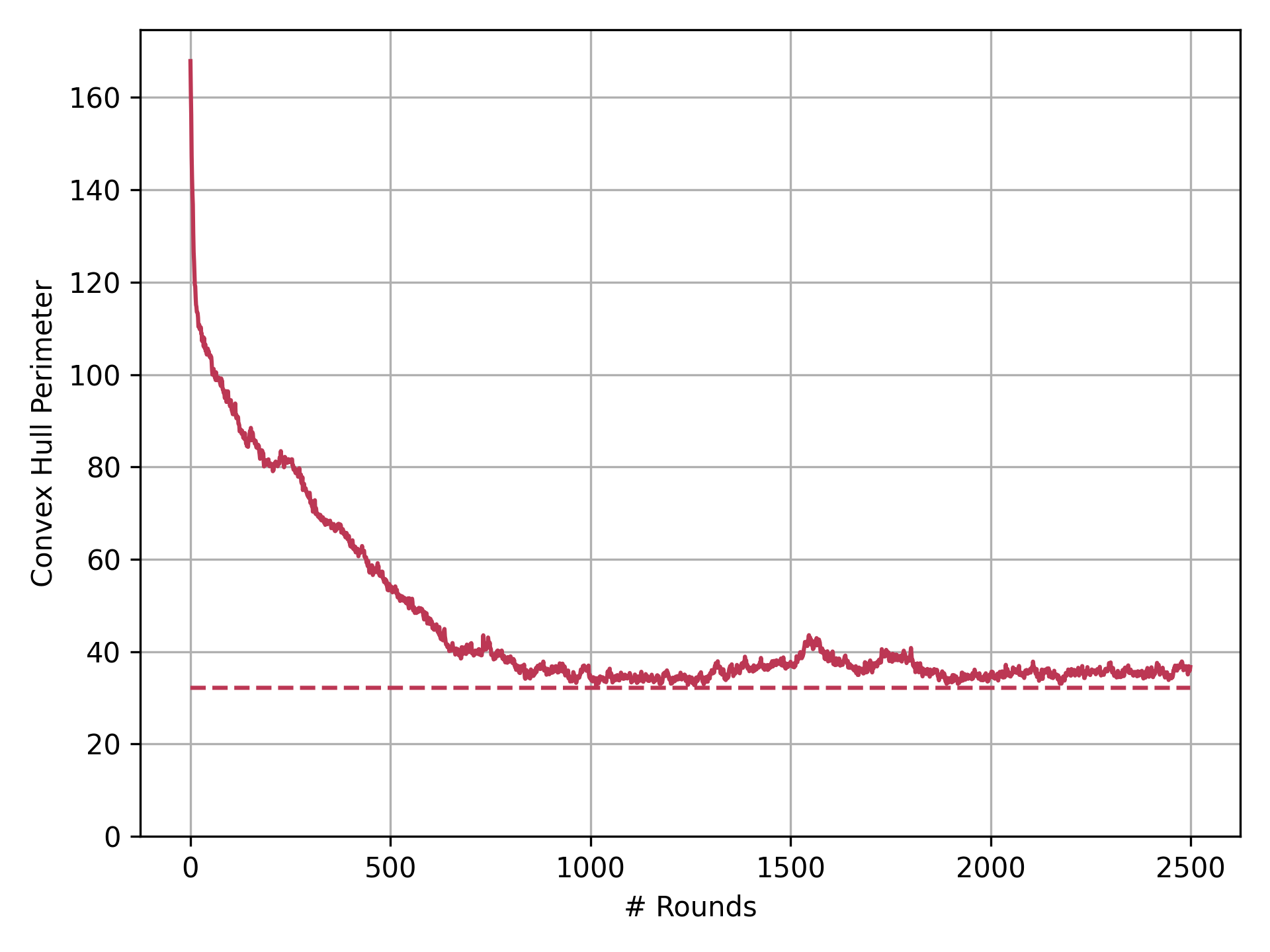}
        \caption{\centering Convex Hull Perimeter}
        \label{fig:metrics:convex}
    \end{subfigure}\\ \medskip
    \begin{subfigure}{.49\textwidth}
        \centering
        \includegraphics[width=\textwidth]{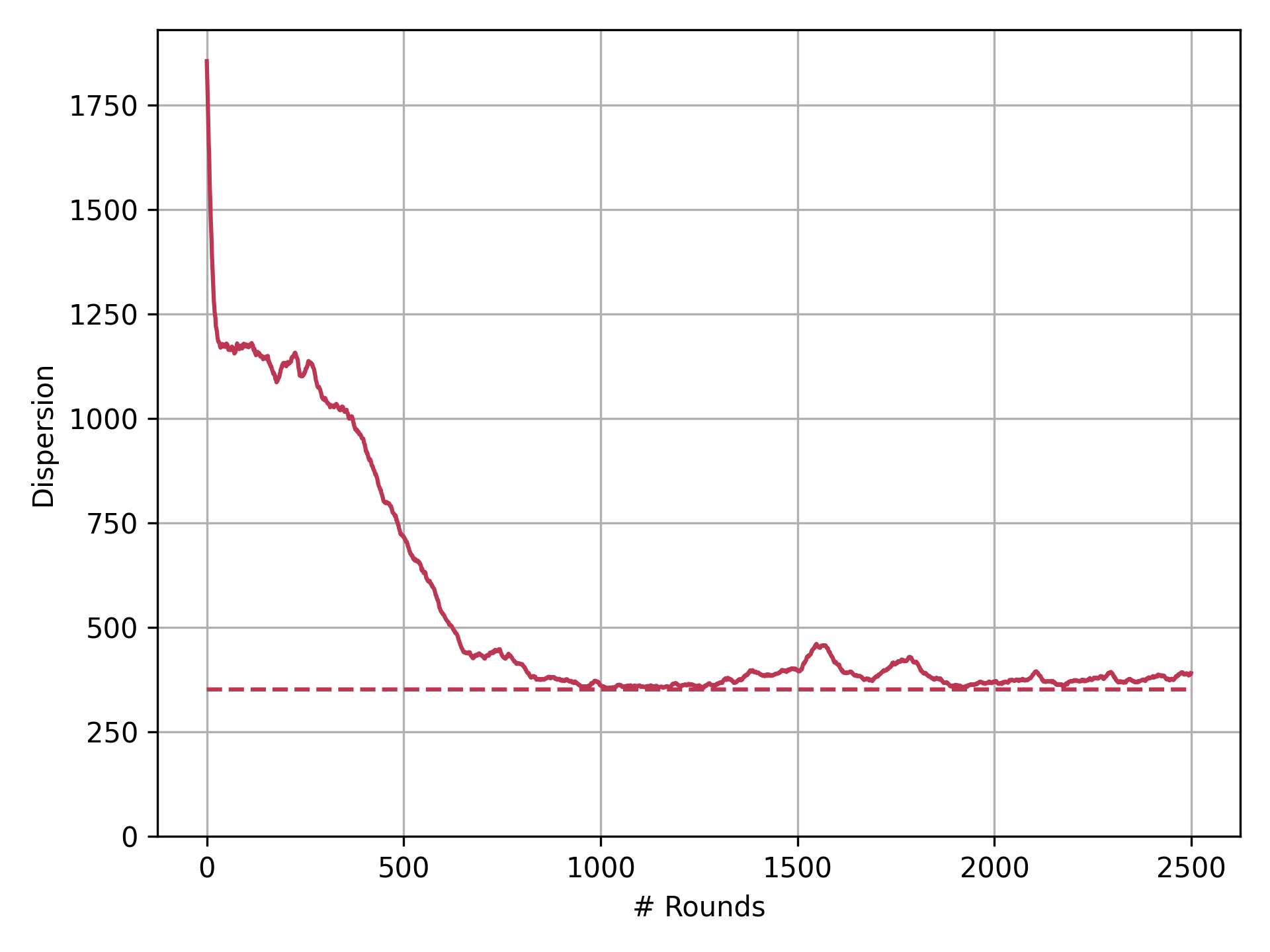}
        \caption{\centering Dispersion}
        \label{fig:metrics:dispersion}
    \end{subfigure}
    \hfill
    \begin{subfigure}{.49\textwidth}
        \centering
        \includegraphics[width=\textwidth]{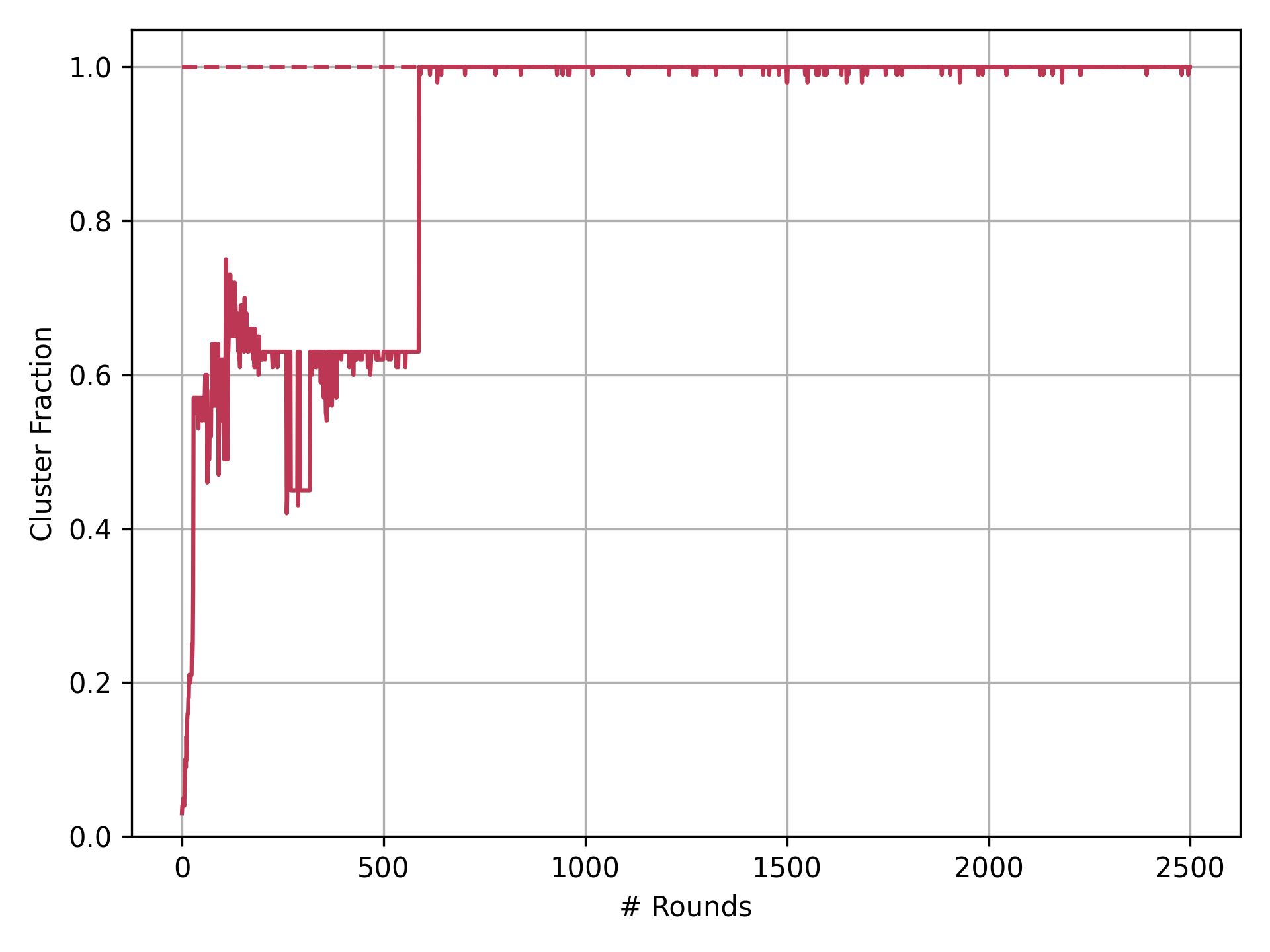}
        \caption{\centering Cluster Fraction}
        \label{fig:metrics:cluster}
    \end{subfigure}
    \caption{Time evolutions of the four aggregation metrics for the same execution of the discrete adaptation by a system of $n = 100$ robots using deadlock perturbation with $d^* = 4$ for 2500 rounds.
    Dashed lines indicate the optimal value for each aggregation metric given the number of robots $n$.}
    \label{fig:metrics}
\end{figure}

Using AmoebotSim~\cite{Daymude2021-amoebotsim}, we simulated this discrete adaptation for a range of noise strengths.
\figtext~\ref{fig:metrics} shows the four aggregation metrics over time for a single run on a system of $n = 100$ robots using deadlock perturbation with $d^* = 4$.
As in the continuous setting (\figtext~\ref{fig:timeevols}), the system progresses steadily but non-monotonically towards aggregation.
However, the discrete adaptation shows significantly more variation, with rapid initial progress towards aggregation followed by slower convergence that includes occasional disaggregation.

\begin{figure}
    \centering
    \begin{subfigure}{.49\textwidth}
        \centering
        \includegraphics[width=\textwidth]{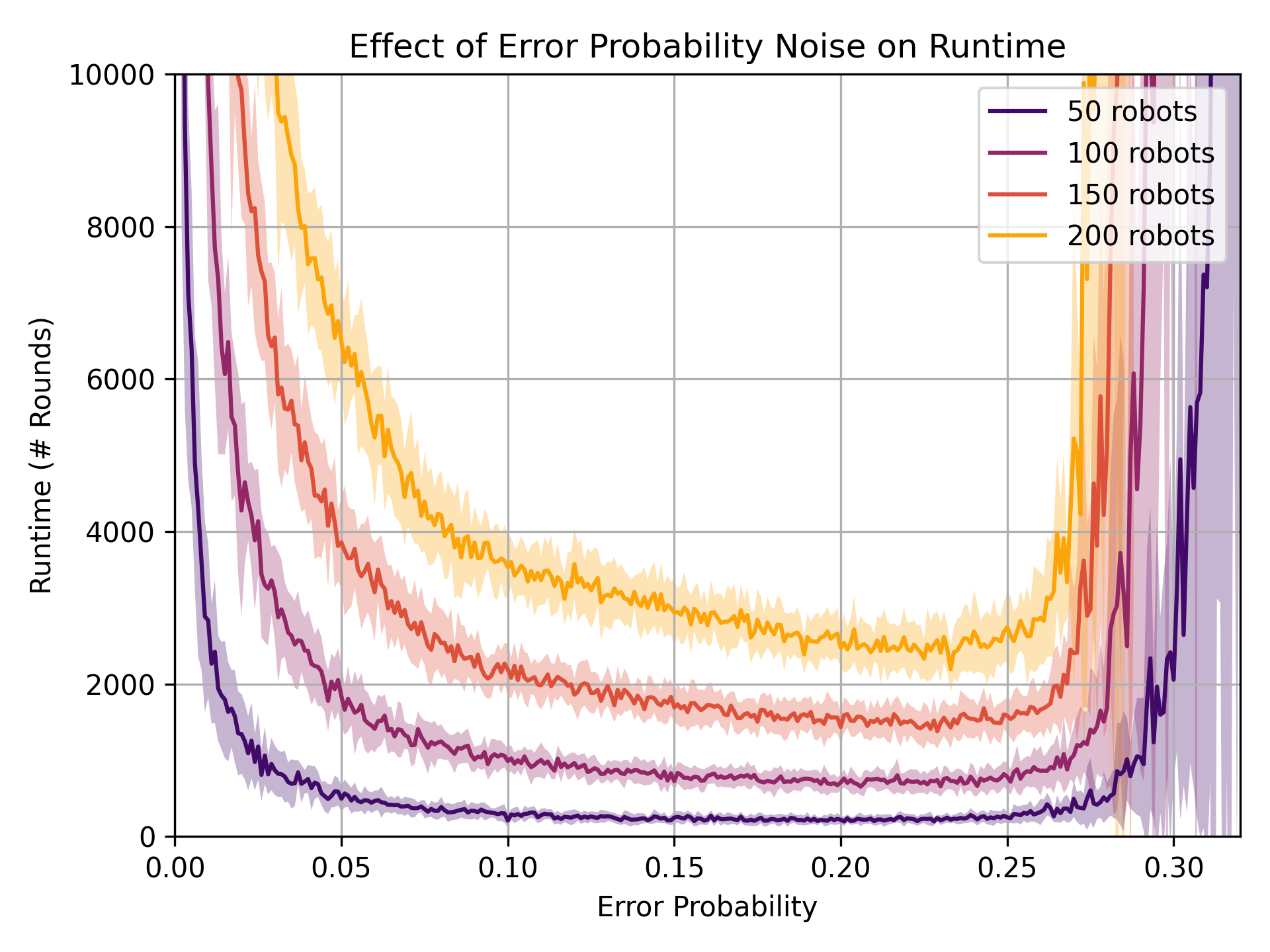}
        \caption{\centering Error Probability}
        \label{fig:noise:errorprob}
    \end{subfigure}
    \hfill
    \begin{subfigure}{.49\textwidth}
        \centering
        \includegraphics[width=\textwidth]{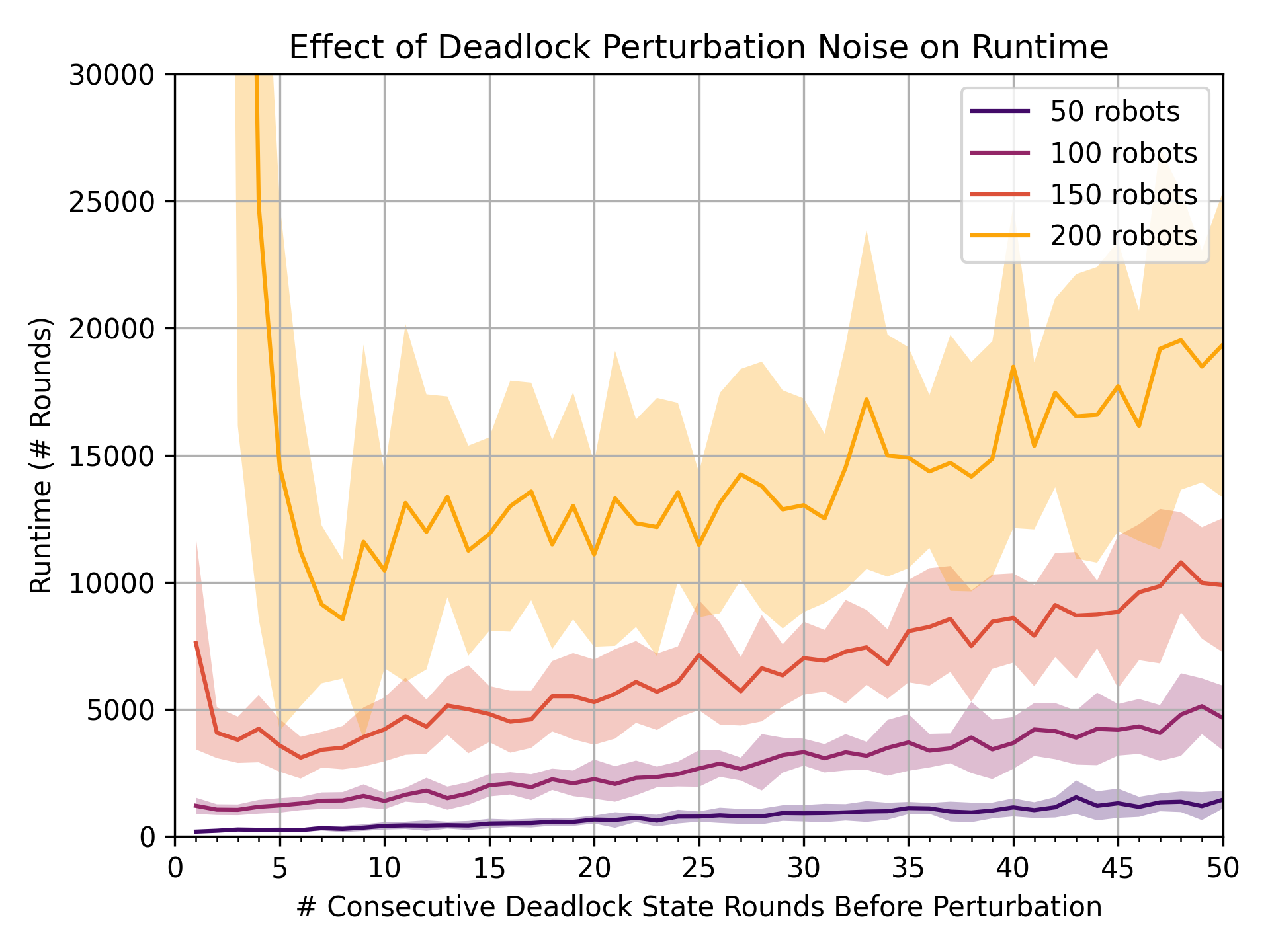}
        \caption{\centering Deadlock Perturbation}
        \label{fig:noise:perturb}
    \end{subfigure}
    \caption{The effects of \textbf{\textsf{(a)}} error probability and \textbf{\textsf{(b)}} deadlock perturbation on the discrete adaptation's time to aggregation for systems of $n = 50$ (purple), $n = 100$ (magenta), $n = 150$ (red), and $n = 200$ (orange) robots.
    Each experiment for a given system size and noise strength was repeated 20 times; average runtime is shown as a solid line and standard deviation is shown as an error tube.
    We consider systems that are within $15\%$ of the minimum dispersion value as aggregated.}
    \label{fig:noise}
\end{figure}

Unlike in the robustness analysis for the continuous setting (Section~\ref{sec:robust}), the discrete adaptation relies on noise to break free from frequently encountered deadlocks.
To determine how much noise the system needs to escape deadlock while still achieving aggregation, we measured the time required for a system to aggregate to within $15\%$ of the minimum dispersion value when using error probability or deadlock perturbation.
For error probability (\figtext~\ref{fig:noise:errorprob}), we observe slow convergence times for both very weak noise ($0.0 < p \leq 0.05$) and very strong noise ($p \geq 0.27$).
Intuitively, very small error probabilities may be insufficient for the robots to escape deadlock while large error probabilities cause the robots to resemble random motion instead of an ordered process progressing towards aggregation.
The intermediate range $0.1 \leq p \leq 0.25$ yields stable and rapid performance across system sizes, though the stable region is wider and more consistent for smaller systems.
Perhaps most interesting is that $p \approx 0.22$ yields near-optimal performance across system sizes; further investigation is needed to determine if this is a critical point independent of system size.

Deadlock perturbation also has a Goldilocks effect, with very strong noise (corresponding to small $d^*$ and frequent perturbations) and very weak noise (corresponding to large $d^*$ and infrequent perturbations) both yielding slow convergence while intermediate values of $d^*$ converge to aggregation the fastest (\figtext~\ref{fig:noise:perturb}).
Intuitively, runtime when using deadlock perturbation scales linearly with the perturbation delay $d^*$, since the longer robots wait before perturbing when blocked, the slower the progress towards aggregation.
Interestingly, a comparison between \figstext~\ref{fig:noise:errorprob} and~\ref{fig:noise:perturb} reveals that deadlock perturbation is less efficient than error probability with respect to runtime: for $n = 200$ robots, the optimal deadlock perturbation value $d^* \approx 8$ yields an average runtime that is over $3\times$ slower than that of the optimal error probability $p \approx 0.22$.


\section{Conclusion} \label{sec:conclude}

In this paper, we investigated the Gauci et al.\ swarm aggregation algorithm~\cite{Gauci2014-aggregation} which provably aggregates two robots and reliably aggregates larger swarms in experiment using only a binary line-of-sight sensor and no arithmetic computation or persistent memory.
We answered the open question of whether the algorithm guarantees aggregation for systems of $n > 2$ robots negatively, identifying how deadlock can halt the system's progress towards aggregation.
In practice, however, the physics of collisions and slipping work to the algorithm's advantage in avoiding deadlock; moreover, we showed that the algorithm is robust to small amounts of noise in its sensors and in its motion.
Next, we considered a generalization of the algorithm using cone-of-sight sensors, proving that for the situation of one moving robot and one static robot, the time to aggregation is improved by a linear factor over the original line-of-sight sensor.
Simulation results showed that small cone-of-sight sensors can also improve runtime for larger systems, though with diminishing returns.
Finally, in an effort to formally prove convergence of the algorithm when noise is explicitly modeled as a mechanism for breaking deadlock, we implemented a noisy, discrete adaptation that was shown to exhibit qualitatively similar behavior to the original, continuous algorithm.
However, the discrete adaptation exhibits occasional disaggregation and other fluctuations during its convergence period, complicating analysis techniques relying on consistent progress towards the goal state (e.g., a potential function argument).
It is possible that a proof showing convergence \textit{in expectation} can be derived, but we leave this for future work.

\bibliographystyle{plainurl}
\bibliography{ref}

\appendix

\section{Deriving a Lower Bound on a Robot's Angle of Rotation} \label{app:lowerbound}

In this appendix, we derive a lower bound on $\gamma$ --- the angle from the cone-of-sight axis to the first line intersecting $\boldsymbol{p}_i$ that is tangent to robot $j$ --- in terms of $\beta$, the fixed size of the cone-of-sight sensor (see \figtext~\ref{fig:cone}).
This lower bound is used in the proof of Theorem~\ref{thm:coneofsight} that shows a linear speedup in time to aggregation when using a cone-of-sight sensor over a line-of-sight sensor.
A trivial lower bound on $\gamma$ is $\gamma \geq 0$ which is achieved when $\beta = 0$ (i.e., when using a line-of-sight sensor).
Similarly, a trivial upper bound on $\gamma$ is $\gamma \leq \beta / 2$ which is achieved when $r_i = 0$ (i.e., when the robots are points as opposed to discs).

\begin{figure}
    \centering
    \includegraphics[width=.6\textwidth]{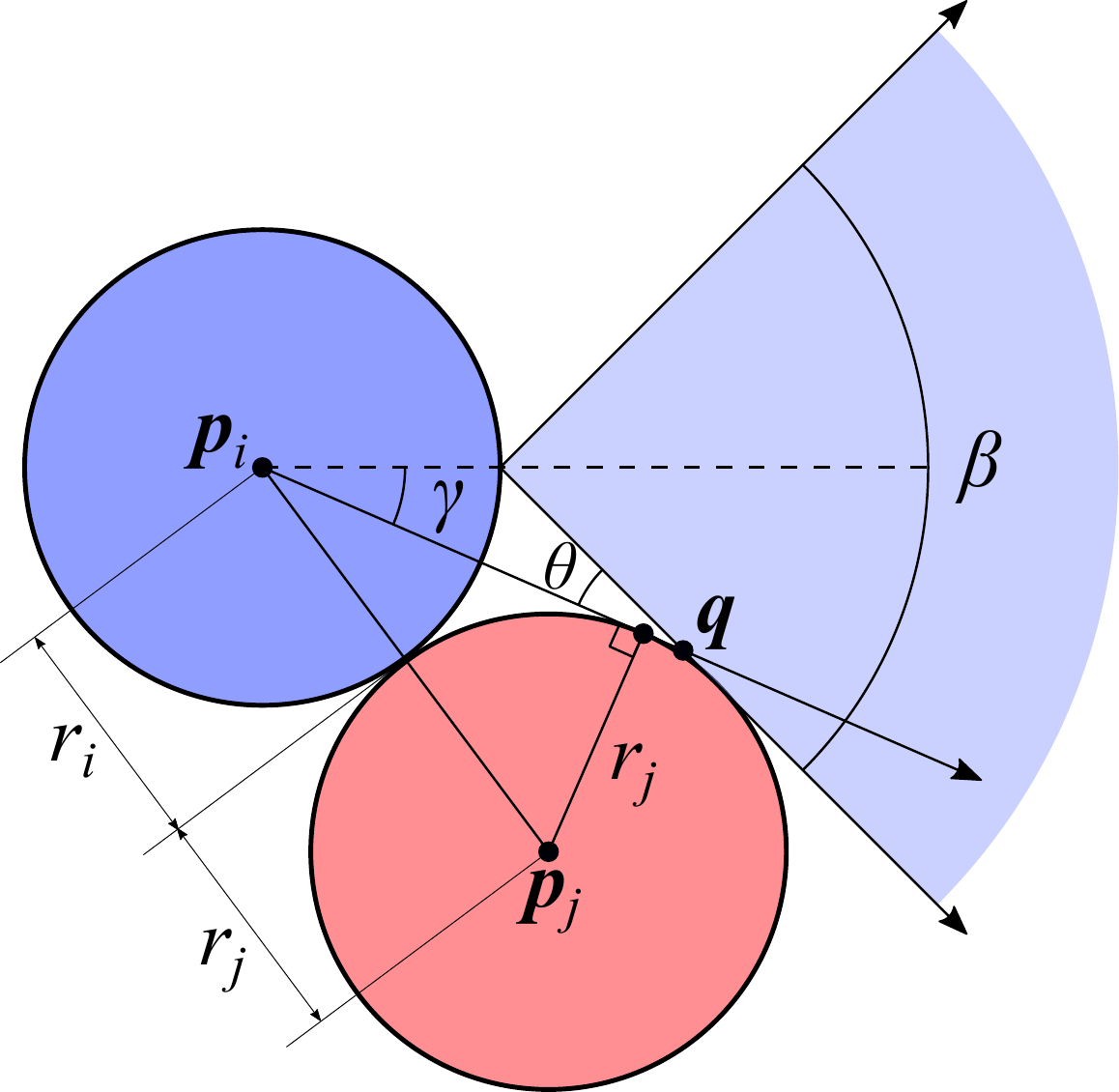}
    \caption{The setup considered in the derivation of the lower bound on $\gamma$.
    Robot $i$ has a cone-of-sight sensor of size $\beta$ and is touching another robot $j$. 
    Point $\boldsymbol{q}$ is the point of intersection between the two tangent lines of robot $j$.}
    \label{fig:lowerbound}
\end{figure}

To obtain a lower bound on $\gamma$ in terms of $\beta$, consider the scenario depicted in \figtext~\ref{fig:lowerbound} where the two robots are touching.
It is easy to see that this scenario yields the minimum value for $\gamma$ without the two robots overlapping.
Letting $z = ||\boldsymbol{q} - \boldsymbol{p}_i||$, the law of sines yields
\[\frac{\sin(\theta)}{r_i} = \frac{\sin(\pi - \beta/2)}{z} = \frac{\sin(\pi)\cos(\beta/2) - \cos(\pi)\sin(\beta/2)}{z} = \frac{\sin(\beta/2)}{z},\]
implying that $\sin(\theta) = r_i\sin(\beta/2) / z$.
Recall that we assumed $r_i = r_j$.
Thus, when robots $i$ and $j$ are touching, $||\boldsymbol{p}_j - \boldsymbol{p}_i|| = 2r_i$.
In this situation, it can be seen that $z \geq r_i\sqrt{3}$, so
\[\sin(\theta) \leq \frac{r_i\sin(\beta/2)}{r_i\sqrt{3}} = \frac{\sin(\beta/2)}{\sqrt{3}}.\]
Since $\gamma = \beta/2 - \theta$,
\[\gamma \geq \beta/2 - \arcsin\left(\frac{\sin(\beta/2)}{\sqrt{3}}\right)\]
which holds because $\theta \leq \arcsin\left(\sin(\beta/2) / \sqrt{3}\right)$.
The ratio
\[\frac{\gamma}{\beta} \geq \frac{\beta/2 - \arcsin\left(\frac{\sin(\beta/2)}{\sqrt{3}}\right)}{\beta}\]
has a positive first derivative with respect to $\beta \in (0, \pi)$, so it is minimized as $\beta \to 0$:
\[\lim_{\beta \to 0} \left(\frac{\beta/2 - \arcsin\left(\frac{\sin(\beta/2)}{\sqrt{3}}\right)}{\beta}\right) = \lim_{\beta \to 0} \left(\frac{1}{2} - \frac{\cos(\beta/2)}{2\sqrt{3} \cdot \sqrt{1-\frac{\sin^2(\beta/2)}{3}}}\right) = \frac{1}{2} - \frac{1}{2\sqrt{3}}.\]
This yields our desired lower bound: $\gamma \geq (1 - 1/\sqrt{3}) \cdot \beta/2$.

\end{document}